\documentclass{article}

\usepackage{microtype}
\usepackage{graphicx}
\usepackage{subfigure}
\usepackage{booktabs} 

\usepackage{hyperref}

\usepackage[accepted]{icml2024}

\usepackage{amsmath}
\usepackage{amssymb}
\usepackage{mathtools}
\usepackage{amsthm}

\usepackage[capitalize,noabbrev]{cleveref}

\theoremstyle{plain}
\newtheorem{theorem}{Theorem}[section]
\newtheorem{proposition}[theorem]{Proposition}
\newtheorem{lemma}[theorem]{Lemma}
\newtheorem{corollary}[theorem]{Corollary}
\theoremstyle{definition}
\newtheorem{definition}[theorem]{Definition}

\theoremstyle{remark}

\usepackage[textsize=tiny]{todonotes}

\icmltitlerunning{Eluder-based Regret for Stochastic Contextual MDPs}
\input{mathdefs}

\newcommand{\D}{\mathcal{D}} 
\newcommand{\C}{\mathcal{C}} 
\newcommand{\F}{\mathcal{F}} 
\newcommand{\Fp}{\mathcal{P}} 
\newcommand{\M}{\mathcal{M}} 
\newcommand{\Hist}{\mathbb{H}} 
\newcommand{\Mhat}{\widehat{\mathcal{M}}}

\newcommand{\betaP}{\beta_P}
\newcommand{\betaR}{\beta_r}

\newcommand{\Esq}[1][\pi_i]{
\mathcal{E}^t_{\mathrm{sq}}(#1,c)}

\newcommand{\dHE}{d_{\mathrm{E}}}
\newcommand{\X}{\mathcal{X}}

\newcommand{\Ehelt}[1][\pi_i]{
\mathcal{E}^t_{\mathrm{H}}(#1,c)}

\newcommand{\Eheli}[1][\pi_i]{
\mathcal{E}^i_{\mathrm{H}}(#1,c)}

\newcommand{\Y}{\mathcal{Y}}

\newcommand{\dEP}{d_{\Fp}}

\usepackage{enumitem}

\usepackage{crossreftools}

\usepackage{algorithm}
\usepackage[noend]{algpseudocode}

 \usepackage{xspace}
 \usepackage{selectp}

\begin{document}

\twocolumn[
\icmltitle{Eluder-based Regret for Stochastic Contextual MDPs}



\icmlsetsymbol{equal}{*}

\begin{icmlauthorlist}
\icmlauthor{Orin Levy}{equal,yyy}
\icmlauthor{Asaf Cassel}{equal,yyy}
\icmlauthor{Alon Cohen}{sch,comp}
\icmlauthor{Yishay Mansour}{equal,yyy,comp}
\end{icmlauthorlist}

\icmlaffiliation{yyy}{Balavatnick school of Computer Science, Tel Aviv University, Tel Aviv, Israel}
\icmlaffiliation{comp}{Google Research, Tel Aviv, Israel}
\icmlaffiliation{sch}{School of Electrical Engeneering, Tel Aviv University, Tel Aviv, Israel}
\icmlcorrespondingauthor{Orin Levy}{orinlevy@mail.tau.ac.il}
\icmlcorrespondingauthor{Asaf Cassel}{acassel@mail.tau.ac.il}
\icmlcorrespondingauthor{Alon Cohen}{alonco@mail.tau.ac.il}
\icmlcorrespondingauthor{Yishay Mansour}{mansour.yishay@gmail.com}
\icmlkeywords{Machine Learning, Reinforcement Learning Theory, Contextual MDPs, Eluder Dimension}

\vskip 0.3in
]



\printAffiliationsAndNotice{\icmlEqualContribution} 

\begin{abstract}
    We present the E-UC$^3$RL algorithm for regret minimization in Stochastic Contextual Markov Decision Processes (CMDPs). The algorithm operates under the minimal assumptions of realizable function class and access to \emph{offline} least squares and log loss regression oracles.
    Our algorithm is efficient (assuming efficient offline regression oracles) and enjoys 
    a regret guarantee of
    $
    \widetilde{O}(H^3 \sqrt{T |S| |A|d_{\mathrm{E}}(\mathcal{P}) \log (|\mathcal{F}| |\mathcal{P}|/ \delta) )})
    $
    , with $T$ being the number of episodes, $S$ the state space, $A$ the action space, $H$ the horizon, $\mathcal{P}$ and $\mathcal{F}$ are finite function classes used to approximate the context-dependent dynamics and rewards, respectively, and $d_{\mathrm{E}}(\mathcal{P})$ is the Eluder dimension of $\mathcal{P}$ w.r.t the Hellinger distance.
    To the best of our knowledge, our algorithm is the first efficient and rate-optimal regret minimization algorithm for CMDPs that operates under the general offline function approximation setting. In addition, we extend the Eluder dimension to general bounded metrics which may be of independent interest.
\end{abstract}


\section{Introduction}


Reinforcement Learning (RL) is a field of machine learning that pertains to sequential decision making under uncertainty.
At the heart of RL is the Markov Decision Process (MDP), a fundamental mathematical model that has been studied extensively.
An agent repeatedly interacts with an MDP by observing its state $s \in S$ and selecting an action $a \in A$, which leads to a new state $s'$ and an instantaneous reward that reflects the quality of the action taken.
The agent's goal is to maximize her return during each episode of interaction with the MDP.
MDPs can be applied to a wide range of real-life scenarios, including advertising, healthcare, games, robotics \citep[see, e.g., ][]{Sutton2018,MannorMT-RLbook}.


Many modern applications involve the presence of additional side information, or \emph{context}, that impacts the environment.
A naive approach to handling the context is to extend the state space of the environment to include it. 
However, this method increases the complexity of learning and policy representation.
Contextual MDPs (CMDPs) offer a more efficient solution by keeping the state space small and treating the context as additional side-information that the agent observes at the start of each episode. 
Furthermore, there exists a mapping from each context to an MDP, and the optimal policy for a given context is the optimal policy of the corresponding MDP~\citep{hallak2015contextual}.
An example of a context is user information that remains constant throughout the episode. 
Such information may include the user's age and interests, and can deeply impact decision making.
This feature makes CMDPs an excellent model for recommendation systems.

As is common in recent works, the aforementioned mapping from context to MDP is assumed to be taken from a known function class, and access to the function class is provided via an optimization oracle.
A distinctive feature between works is whether they assume access to \emph{online} or \emph{offline} oracles. 
Intuitively, in both settings we have a function class $\F = \brk[c]{f : X \to Y}$, a loss $\ell: Y \times Y \to \R$, and a dataset\footnote{We think of $X$ as the context and $Y$ as the MDP.} $\brk[c]{(x_i,y_i)}_{i=1}^{n}$. An offline oracle observes the entire data and needs to find $\hat{f}_\star \in \argmin_{f \in \F}\sum_{i=1}^{n} \ell(f(x_i),y_i)$. 
An online oracle makes a sequence of predictions $f_1, \ldots, f_n$ where $f_i$ can depend on data up to $i-1$, and its goal is to minimize regret, given by $\sum_{i=1}^{n} \ell(f_i(x_i),y_i) - \ell(\hat{f}_\star(x_i),y_i)$.
The offline problem can potentially be easier to solve than the online problem. 
Moreover, practical deep RL applications typically work in the offline regime.
%

Previously, \citet{modi2020no} obtained $\widetilde{O}(\sqrt{T})$ regret for a generalized linear model (GLM).
\citet{foster2021statistical} obtain $\widetilde{O}(\sqrt{T})$ regret for general function approximation and adversarially chosen contexts, assuming access to a much stronger online estimation oracle.
However, they noted the challenge of implementing their methodology using offline oracles.
It thus remained open whether, for stochastic contexts, we can restrict the access to offline oracles.
Recently, \citet{levy2022optimism} gave an $\widetilde{O}(\sqrt{T}/p_{min})$ regret algorithm for stochastic contexts using offline least squares regression,
where $p_{min}$ is a minimum reachability parameter of the CMDP.
This parameter can be arbitrarily small and in general CMDPs leads to an
%
$\widetilde{O}(T^{3/4})$ regret guarantee.
The question of whether the minimum reachability assumption can be obviated or replaced by a less restrictive assumption on the function class remained open. 

In this work, we give the first $\smash{\widetilde{O}(\sqrt{T})}$ regret algorithm for stochastic contexts using standard offline oracles, under the bounded Eluder dimension~\citep{russo2013eluder} assumption (more details in~\cref{subsec:eluder}).

\noindent\textbf{Contributions.}
We present the E-UC$^3$RL algorithm for stochastic CMDPs with \emph{offline} regression oracles. Our algorithm is efficient (assuming efficient oracles) and enjoys an
$
\widetilde{O}\big(H^3 \linebreak[1]\sqrt{T \abs{S} \abs{A} \dHE(\Fp)\log \brk{\abs{\F}\abs{\Fp}/ \delta} }\big)
$
regret bound
with probability at least $1-\delta$, where $S$ is the state space, $A$ the action space, $H$ the horizon, $\Fp$ and $\F$ are finite function classes used to approximate the context-dependent dynamics and rewards, respectively, and $\dHE(\Fp)$ is the Eluder dimension with respect to Hellinger distance of the context-dependent dynamics function class $\Fp$.
The algorithm builds on the ``optimistic in expectation'' approach of~\citet{levy2022optimism} but modifies it with a log-loss oracle for the dynamics approximation and carefully chosen counterfactual reward bonuses.
To that end we present an extension of the Eluder dimension to general bounded metrics (rather than the $\ell_2$ norm considered by~\citet{russo2013eluder} and \citet{osband2014model}).
An additional key technical tool enabling our result is a multiplicative change
of measure inequality for the value function. Both tools may be of separate interest.
%


\noindent\textbf{Comparison with \citet{levy2022optimism}.} This work is most closely related to ours. It relies on a minimum reachability assumption and provides a regret bound of $\widetilde{O}(\sqrt{T}/p_{min})$ where $p_{min}$ is the reachability parameter of the CMDP. This implies that any policy $\pi$ will reach any state $s$ with a probability of at least $p_{min}$ for any context $c$, hence $p_{min} \le 1 / \abs{S}$. As such, any policy inherently explores with probability $p_{min}$, significantly simplifying the exploration task. 
While the notion of minimum reachability is intuitive, it fails even for deterministic transition functions where $p_{min} = 0$. Moreover, it is impossible to estimate it online as we typically observe each context only once.

The primary focus of our work is to replace minimum reachability, which is a structural assumption about the true CMDP, with an assumption about the dynamics function class, which is chosen by the learner. This makes learning an effective exploratory policy non-trivial, necessitating innovative confidence bounds that capture the intricacies of the function class learning complexity. In our approach, we employ the Eluder dimension as the complexity measure. 
One can show that minimum reachability implies a bound on the Eluder dimension, but, the Eluder dimension can be much smaller. 

\noindent\textbf{The general function approximation literature.}
We stress that the role of the Eluder dimension in this work is to avoid direct dependence on the size of the context space, which could be prohibitively large, while also maintaining computational efficiency. This is unlike previous works on function approximation in RL (see, e.g., \citet{jiang2017contextual,jin2021BellmanEluder,pac-uniform-with-eluder,chen2022general,wang2020reinforcement,dann2021provably,liu2022partially}) that use an Eluder dimension to avoid dependence on $\abs{S}, \abs{A}$. These works are often computationally inefficient and  require additional structural assumptions regarding the MDP, such as low Bellman-rank or Bellman completeness, or a much stronger optimization~oracle. 


%


\noindent\textbf{Additional Related Work.}
%
\citet{hallak2015contextual} were the first to study regret guarantees in the the CMDP model. However, they assume a small context space, and their regret is linear in its size. 
\citet{jiang2017contextual} present OLIVE, a sample efficient algorithm for learning Contextual Decision Processes (CDP) under the low Bellman rank assumption. In contrast, we do not make any assumptions on the Bellman rank.
\citet{sun2019model} use the Witness Rank to derive PAC bounds for model based learning of CDPs.  
\citet{modi2018markov} present generalization bounds for learning \emph{smooth} CMDPs and finite contextual linear combinations of MDPs. 
\citet{modi2020no} present a regret bound of $\widetilde{O}(\sqrt{T})$ for  CMDPs under a Generalized Linear Model (GLM) assumption.
Our function approximation framework is more general than smooth CMDPs or GLM.

\citet{foster2021statistical} 
present the Estimation to Decision (E2D) meta algorithm and apply it to obtain $\widetilde{O}(\sqrt{T})$ regret for adversarial Contextual RL. Later,~\citet{xie2022role} show sample complexity bounds for online reinforcement learning using online oracle, that can be also applied to CMDPs.
\citet{levy2023adversarialCMDPs} obtained similar results using their OMG-CMDP! algorithm. However, these works assume access to online estimation oracles and their bounds scale with the oracle's regret. In contrast, we use substantially weaker offline regression oracles. 
It is not clear to us whether \citeauthor
{foster2021statistical}'s Inverse Gap Minimization (IGM) technique or \citeauthor
{levy2023adversarialCMDPs}'s convex optimization with log-barrier method can be applied to CMDPs with offline regression oracles.


\citet{levy2022learning} study the sample complexity of learning CMDPs using function approximation and provide the first general and efficient reduction from CMDPs to offline supervised learning.
However, their sample complexity scales as $\smash{\epsilon^{-8}}$, and thus they cannot achieve the optimal $\smash{\sqrt{T}}$ rate for the regret.
%
%
%
\citet{levy2022optimism}, previously mentioned here in relation to upper bounds, also showed an $\Omega(\sqrt{TH|S||A|\log(|\mathcal{G}|/|S|)/\log(|A|)})$ regret lower bound for the general setting of offline function approximation with $\abs{\mathcal{G}}$, the size of the function class used to approximate the rewards in each state.

More broadly, CMDPs are a natural extension of the extensively studied Contextual Multi-Armed Bandit (CMAB) model.
CMABs augment the Multi-Arm Bandit (MAB) model with a context that determines the rewards \citep{MAB-book,Slivkins-book}. \citet{Langford2007,agarwal2014taming} use an optimization oracle to derive an optimal regret bound that depends on the size of the policy class they compete against. 
Regression based approaches were presented in \citet{agarwal2012contextual,foster2020beyond,foster2018practical,foster2021efficient,simchi2021bypassing,zhang2022feel}.
Most closely related to our work, \citet{xu2020upper} present the first optimistic algorithm for CMAB. They assume access to a least-squares regression oracle and achieve $\widetilde{O}(\sqrt{T |A| \log |\F|})$ regret, where $\F$ is a finite and realizable function class, used to approximate the rewards.
Extending their techniques to CMDPs necessitates accounting for the context-dependent dynamics whose interplay with the rewards significantly complicates decision making. This is the main challenge both in our work and in \citet{levy2022optimism}.


%

The Eluder dimension was introduced by \citet{russo2013eluder} and applied to derive sublinear regret for MABs and CMABs. \citet{osband2014model} showed an application of the Eluder dimension to derive a regret bound for model-based reinforcement learning and \citet{wen2017efficient} for deterministic systems.
\citet{Wang2020RLEluder} use it to derive a regret bound for value function approximation. 
\citet{jin2021BellmanEluder} present the Bellman-Eluder dimension and use it to develop sample-efficient algorithms for a family of RL problems where both the Bellman rank and the Eluder dimension are low.
\citet{Ayoub2020Value-Targeted-Eluder} apply the Eluder dimension to derive a regret bound for tabular episodic RL using targeted value regression. We, on the other hand, extend the Eluder dimension to general bounded metrics and apply it to contextual RL.


\section{Preliminaries}

\subsection{Episodic Loop-Free Markov Decision Process (MDP)}
An MDP is defined by a tuple $(S,A,P,r,s_0, H)$, where $S$ and $A$ are finite sets describing the state and action spaces, respectively;
$s_0\in S$ is the unique start state; $H \in \N$ is the horizon;
$P :  S \times A \times S \to [0,1]$ defines the probability of transitioning to state $s'$ given that we start at state $s$ and perform action $a$; and $r(s,a)$ is the expected reward of performing action $a$ at state $s$.
An episode is a sequence of $H$ interactions where at step $h$, if the environment is at state $s_h$ and the agent performs action $a_h$ then (regardless of past history) the environment transitions to state $s_{h+1} \sim P(\cdot \mid s_h, a_h)$ and the agent receives reward $R(s_h, a_h) \in [0,1]$, sampled independently from a distribution $\D_{s_h,a_h}$ that satisfies $r(s_h, a_h) = \E_{\D_{s_h,a_h}} \brk[s]*{ R(s_h,a_h)}$.

For technical convenience and without loss of generality, we assume that the state space and accompanying transition probabilities have a loop-free (or layered) structure. Concretely, we assume that the state space can be decomposed into $H+1$ disjoint subsets (layers) $S_0, S_1, \ldots, S_{H-1}, S_H$ such that transitions are only possible between consecutive layers, i.e., for $h' \neq h+1$ we have $P(s_{h'} | s_h, a) = 0$ for all $s_{h'} \in S_{h'}, s_h \in S_h, a \in A$.
In addition, $S_H = \{s_H\}$, meaning there is a unique final state with reward $0$.
We note that this assumption 
can always be satisfied by increasing the size of the state space by a factor of $H$.

A \emph{deterministic stationary policy} $\pi : S \to A$ is a mapping from states to actions. 
Given a policy $\pi$ and MDP 
    $
        M 
        = 
        (S,A,P,r,s_0, H)
    $, 
the
$h \in [H-1] $ stage value function of a state $s \in S_h$ is defined as 
    \[
        V^{\pi}_{M,h} (s)
        = 
        \mathbb{E}_{\pi, M} 
        \brk[s]*{
        \sum_{k=h}^{H-1} r(s_k, a_k)\Bigg| s_h = s 
        }
        .
    \]  
For brevity, when $h = 0$ we denote $V^{\pi}_{M,0}(s_0) := V^{\pi}_M (s_0)$, which is the expected cumulative reward under policy $\pi$ and its measure of performance. A  policy $\pi^\star_M$ is \emph{optimal} for MDP $M$ if it satisfies that
$
\pi^\star_M
\in
\argmax_{\pi: S \to A}\{V^{\pi}_{M}(s_0)\}
$.
It is well known that such a policy is optimal even among the class of stochastic and history dependent policies 
(see, e.g., \citealp{puterman2014markov,Sutton2018,MannorMT-RLbook}).

\subsection{Problem Setup: Stochastic Contextual Markov Decision Process (CMDP)}
 A \emph{CMDP} is defined by a tuple $(\mathcal{C},S, A, \mathcal{M})$ where $\mathcal{C}$ is the context space, $S$  the state space and $A$  the action space. The mapping $\mathcal{M}$  maps a context $c\in \mathcal{C}$ to an MDP
    $
        \mathcal{M}(c) 
        =
        (S, A, P^c_\star, r^c_\star,s_0, H)
    $, 
where $r^c_\star(s,a) = \E[R^c_\star(s,a)|c,s,a]$, $R^c_\star(s,a) \sim \D_{c,s,a}$.
We assume that $R^c_\star(s,a) \in [0,1]$.

We consider a \emph{stochastic} CMDP, meaning, the context is stochastic. Formally, we assume that there is an unknown distribution $\mathcal{D}$ over the context space $\mathcal{C}$, and for each episode a context is sampled i.i.d.\ from $\mathcal{D}$.
For mathematical convenience, we assume that the context space $\C$ is finite but potentially very large. Our results do not depend on the size of the context space and can be further extended to infinite context spaces.

%
%

A deterministic  \emph{context-dependent policy} $\pi : \mathcal{C} \times S \to A$ maps a context $c \in \mathcal{C}$ to a policy $\pi(c;\cdot) : S \to A$.    
Let $\Pi_\C$ denote the class of all deterministic context-dependent policies.

\noindent\textbf{Interaction protocol.} The interaction between the agent and the environment is defined as follows.
    In each episode $t=1,2,...,T$ the agent:
    \begin{enumerate}[nosep,label=(\roman*),leftmargin=*]
        \item Observes context $c_t \sim \D$;
        \item Chooses a policy $\pi_t$ (based on $c_t$ and the observed history);
        \item Observes trajectory
        $
        (c_t, s^t_0, a^t_0, r^t_0, \ldots, 
        s^t_H)$ generated by playing
        $\pi_t$  in $\M(c_t)$.
    \end{enumerate}
%
%
%
%
%
Our goal is to minimize the regret, defined as 
$$
\Regrv_T 
    := \sum_{t=1}^T 
    V^{\pi^\star(c_t;\cdot)}_{\M(c_t)}(s_0)
    -
    V^{\pi_t(c_t;\cdot)}_{\M(c_t)}(s_0)
    ,
$$ 
where
$
    \pi^\star \in \Pi_\C
$
is an optimal context-dependent policy.
We aim to derive regret bounds that are independent of the context space size $\abs{\C}$. For that purpose, we make function approximation assumptions in \cref{subsec:function-class-aassumption}, which rely on the following definition of Eluder dimension.

\subsection{Metric Eluder Dimension}\label{subsec:eluder}
We extend the notion of Eluder dimension, given by \citet{osband2014model}, to general bounded metrics.
Let $\X$ be a set and $(\Y, D)$ a bounded metric space. 
Let $\Fp \subseteq \brk[c]{\X \to \Y}$ be a set of functions from $\X$ to $\Y$.
We say that $x \in \X$ is $(D,\epsilon)-$dependent of $x_1, \ldots, x_n$ if and only if for any $P, P' \in \Fp$ it holds that 
\begin{align*}
    \sum_{i=1}^{n} D^2(P(x_i), P'(x_i))
    \le
    \epsilon^2
    \implies
    D^2(P(x), P'(x))
    \le
    \epsilon^2
    .
\end{align*}
We say that $x \in \X$ is $(D,\epsilon)-$independent of $x_1, \ldots, x_n$ if it is not $(D,\epsilon)-$dependent.

\begin{definition}[Metric-Eluder Dimension]\label{def:Metric-Eluder}
We say that $d := \dHE(\Fp, D, \epsilon)$ is the $(D, \epsilon)-$Eluder dimension of a class $\Fp$ if $d$ is the maximum length of sequences $x_1,\ldots, x_d$ and $\epsilon'_1,\ldots,\epsilon'_d$ such that for all $1 \le i \le d$, $x_i$ is $(D,\epsilon'_i)-$independent of its prefix $x_1, \ldots , x_{i-1}$ and $\epsilon'_i \ge \epsilon$.
\end{definition}
This quantity roughly corresponds to the number of queries required to $\epsilon$ identify a function in $\Fp$. The utility of this definition is summarized in the following result, which is a straightforward adaptation of
Proposition 6 in \citet{osband2014model} (proof in \cref{sec:eluder-proofs}).
%
For any $\Fp' \subseteq \Fp$, define its radius at $x \in \X$ as
$
    w_{\Fp'}(x)
    =
    \sup_{P,P' \in \Fp'} D({P}(x), {P'}(x))
$.

\begin{lemma}
\label{lemma:russo-lemma2}
    For any $t \in [T], h \in [H]$ let $x_h^t \in \X$ and ${P}_t \in \Fp$ be arbitrary.
    Define the confidence sets with parameter $\beta$ as
$$
    \Fp_t
    =
    \brk[c]*{
    P\in \Fp
    \;:\;
    \sum_{i=0}^{t-1} \sum_{h=0}^{H-1} D^2(P(x_h^i), {P}_t(x_h^i)
    )
    \le
    \beta
    }
    .
$$
    We have that
    \begin{align*}
        \sum_{t=1}^{T}\sum_{h=0}^{H-1} \brk{w_{\Fp_t}(x^t_h)}^2
        \le
        6 \dHE(\Fp, D, T^{-1/2}) \beta \log T
        .
    \end{align*}
\end{lemma}

\subsection{Function Class Assumptions}\label{subsec:function-class-aassumption}

We note that, without further assumptions, 
the regret may scale linearly in the size of the context space \citep{hallak2015contextual}. Even worse, if the context space contains more than $T$ contexts, and the distribution over the contexts is uniform, the regret may scale linearly in $T$. 
We overcome this limitation by imposing the following function approximation assumptions, that extend similar notions in the Contextual Multi-Armed Bandits literature \citep{agarwal2012contextual,russo2013eluder,foster2018practical,foster2021efficient,simchi2021bypassing} to CMDPs.


\noindent\textbf{Realizable reward function approximation.}\label{par:reward-function-approx}
Our algorithm gets as input a finite function class $\F \subseteq \C \times S \times A \to [0,1]$ such that there exists $f_\star \in \F$ that satisfies 
$
   f_\star(c,s,a)
   = 
   r^c_\star(s,a)
$ for all $c \in \C$ and $(s,a) \in S \times A
$.

\noindent\textbf{Realizable dynamics function approximation.}
Our algorithm gets as input a finite function class $\Fp \subseteq  S \times ( S \times A \times \C) \to [0,1]$ such that $P_\star \in \Fp$, and every function $P \in \Fp$ represents valid transition probabilities, i.e., satisfies 
$
    \sum_{s' \in S} P(s' \mid s,a,c)  = 1$ for all $c \in \C$ and $(s,a) \in S \times A
$.
For convenience, we denote $P(s' \mid s,a,c) = P^c(s' \mid s,a)$, for all $P \in \Fp$. 
%






\noindent\textbf{%
Offline regression oracles.}
Given a data set $D = \brk[c]{(c_i, s_i,a_i,s'_i,r_i)}_{i=1}^{n}$, we assume access to
offline oracles that solve the optimization problems:
%
\begin{align*}
    \tag{Least Squares Regression (LSR)}
    \widehat{f}
    \in
    \arg\min_{f \in \F }\sum_{i=1}^n (f(c_i,s_i,a_i) - r_i)^2,
    \\
    \tag{Log Loss Regression (LLR)}
    \widehat{P}
    \in
    \arg\min_{P \in \Fp }\sum_{i=1}^n \log \frac{1}{P^{c_i}(s'_i \mid  s_i, a_i)}.
\end{align*}
Notice that the above problems can always be solved by iterating over the function class. However, since we consider strongly convex loss functions, there are function classes where these optimization problems can be solved efficiently. One particular example is the class of linear functions. 




\noindent\textbf{Eluder Dimension (w.r.t the Squared Hellinger distance).}
As shown by \citet{foster2021statistical}, the log-loss oracle provides generalization guarantees with respect to the squared Helligner distance.
\begin{definition}[Squared Hellinger Distance]
\label{def:hellinger}
    For any two distributions $\mathbb{P}$, $\mathbb{Q}$ over a discrete support $X$, the Squared Hellinger Distance is defined as
    $$
        D^2_H(\mathbb{P}, \mathbb{Q}) 
        := 
        \sum_{x \in X} \brk2{\sqrt{\mathbb{P}(x)} - \sqrt{\mathbb{Q}(x)}}^2
    .
    $$
\end{definition}
\noindent
The Hellinger distance is a bounded metric. 
Thus, we assume a known upper bound $\dEP$ of $\dHE(\Fp, D_H, T^{-1/2})$, the Eluder dimension of $\Fp$ with respect to Hellinger distance.

Clearly, for a finite class the Eluder dimension is at most the number of functions.
For classes of discrete distributions, where the minimum probability is $p>0$, one can bound the Eluder dimension w.r.t. the Hellinger distance by $d_2/p$, where $d_2$ is the (standard) Eluder dimension w.r.t. $\ell_2$.
(See~\cref{lemma:Eluder-upper-bound-p-min}).

\section{Algorithm and Main Result}

We present Eluder Upper Counterfactual Confidence for Contextual Reinforcement Learning (E-UC$^3$RL), given in \cref{alg:UCCRL-main}. At each episode $t$, the algorithm estimates the reward and dynamics using the regression oracles. It then constructs an optimistic CMDP using reward bonuses and plays its optimal policy. The reward bonuses are inspired by the notion of \emph{counterfactual confidence}, suggested by \citet{xu2020upper} for CMABs. The original idea was to calculate the confidence bounds using the counterfactual actions of past policies given the current context. \citet{levy2022optimism} adapted this approach to CMDPs using the minimum reachability assumption, without which, it becomes crucial to also consider counterfactual states. Notice that the states are stochastically generated by the MDP in response to the agent's played actions. This makes counterfactual state computation impossible without access to
the 
true dynamics. Instead, we consider the counterfactual probabilities of a state-action pair and evaluate this quantity using the estimated dynamics. 
These probabilities are typically referred to as \emph{occupancy measures}~\citep{zimin2013online}. Concretely, for any non-contextual policy $\pi$ and dynamics $P$,
let $q_h(s,a \mid \pi, P)$ denote the probability of reaching state $s\in S$ and performing action $a\in A$ at time $h \in [H]$ of an episode generated using policy $\pi$ and dynamics $P$.
Note that, given $\pi$ and $P$, the occupancy measure of any state-action pair can be computed efficiently using a standard planning algorithm.

At round $t$ and $(s,a,h,c)-$tuple the cumulative occupancy measure of past policies, i.e., $\sum_{i=1}^{t-1} q_h(s,a|\pi_i^c, P_\star^c)$ is a good indicator for the quality of the estimated dynamics and rewards $\widehat{f}_t$ and $\widehat{P}_t$. Thus we would ideally choose bonuses inversely proportional to this quantity. Since $P_\star$ is unknown, it is natural to replace it in $q_h(\cdot)$ with the most recent estimate $\widehat{P}_t$. However, the instability of the oracle estimates $\widehat{P}_t$ means that $q_h(s,a,|\pi_i^c, \widehat{P}_t^c)$ can change arbitrarily with $t$, which may lead to overly large bonuses. We resolve this instability using the Eluder dimension assumption, which allows us to replace $\widehat{P}_t$ with $\widehat{P}_i$ in $q_h(\cdot)$, thus stabilizing the occupancy measure estimate of $\pi_i$ as $q_h(s,a,|\pi_i^c, \widehat{P}_i^c)$ for all $t$.
Finally, since our bonuses are based on past context-dependent policies, we first have to compute $\pi_k(c_t;\cdot)$ for all $k \in [t-1]$, which is the purpose of our internal loop~(\cref{ln:loop}).
%
%
%
%
%
%
%
%
%
%
\begin{algorithm}[!ht]
    \caption{Eluder Upper Counterfactual Confidence for Contextual RL (E-UC$^3$RL)
    }
    \label{alg:UCCRL-main}
    \begin{algorithmic}[1]
        \State
        {\textbf{inputs:}\;MDP parameters: 
                 $S$, $A$, $s_0$, $H$;
                 
                \quad\; tuning parameters $\betaR, \betaP$.
        } 
        \For{round $t = 1, \ldots, T$}
            \State{compute using the LSR oracle:
           $$
                \widehat{f}_t \in 
                \arg\min_{f \in \F}
                \sum_{i=1}^{t-1} \sum_{h=0}^{H-1} ( f(c_i,s^i_h,a^i_h) - r^i_h)^2
            $$
            } \label{ln:lsr-orace}
            \State{
            computed using the LLR oracle: 
            $$
                \widehat{P}_t \in 
                \arg\min_{P \in \Fp}
                \sum_{i=1}^{t-1} \sum_{h=0}^{H-1} \log \brk*{ \frac{1}{P^{c_i}(s^i_{h+1}|s^i_h,a^i_h)}}
            $$
            } \label{ln:lrr-oracle}
            \State{observe a fresh context $c_t \sim \D$.}
            \For{$k= 1 ,2, \ldots, t$}\label{ln:loop}
                \State{compute  for all $h \in [H]$ and ${(s,a) \in S_h \times  A}$: 
                \begin{align*}
                    &\widehat{r}_k^{c_t}(s, a) 
                    = 
                    \\
                    &\widehat{f}_k(c_t,s,a)
                     +
                    b^{\betaR}_k(c_t,s,a,h) +
                    H b^{\betaP}_k(c_t,s,a,h)
                \end{align*}
                where
                \begin{align*}
                    \qquad &b^\beta_k(c,s,a,h) 
                    =
                    \\
                    &\min\Bigg(
                    1
                    ,
                    \frac{\beta/2}{1 + \sum_{i=1}^{k-1}q_h(s,a| \pi_i(c;\cdot), \widehat{P}^c_i)}
                    \Bigg)
                    .
               \end{align*}
                }\label{ln:bonuses}
                \State{define $\Mhat_k(c_t) = (S, A, \widehat{P}^{c_t}_k, \widehat{r}^{c_t}_k, s_0, H)$.}
                \State{compute using a planning algorithm:
                \begin{align*}
                \pi_k(c_t;\cdot) \in \arg\max_{\pi :S \to A
                }V^{\pi}_{\Mhat_k(c_t)}(s_0)
                .
                \end{align*}}
            \EndFor{}    
            \State{play $\pi_t(c_t;\cdot)$ and observe a trajectory 
            $
            \sigma^t = 
            {(c_t,\linebreak[1] s^t_0, a^t_0, r^t_0, s^t_1, \ldots, s^t_{H-1}, a^t_{H-1}, r^t_{H-1},s^t_H)}
            $.}
            \EndFor{}
    \end{algorithmic}
\end{algorithm}

The following is our main result for \cref{alg:UCCRL-main}. We sketch its proof in \cref{sec:analysis}, and defer the complete proof to \cref{sec:regret-proof}.
\begin{theorem}[E-UC$^3$RL regret bound]\label{thm:UC3RL-regret-bound}
    For any $T > 1$ and $\delta \in (0,1)$, suppose we run \cref{alg:UCCRL-main} with parameters 
    \begingroup\allowdisplaybreaks
    \begin{align*}
        &\betaR
        =
        \sqrt{\frac{504 T H^2 \dEP \log^2(64 T^4 H \abs{\F} |\Fp|/\delta^2)}{|S||A| \log(T+1)}}
        ,
        \\
        &\betaP
        =
        \sqrt{\frac{2029 T H^2 \dEP \log^2(8 T H |\Fp| /\delta)}{|S||A| \log(T+1)}}
        ,
    \end{align*}
    \endgroup
    and $\dEP\geq \dHE(\Fp, D_H, T^{-1/2})$.
    Then, with probability at least $1-\delta$ it holds that
    \begingroup\allowdisplaybreaks
    \begin{align*}
        \Regrv_T(&\text{E-UC$^3$RL}) 
        \leq
        \widetilde{O}
        \brk{
        H^3
        \sqrt{
        T \abs{S} \abs{A}
        \dEP
        \left(\log \brk{\abs{\F}\abs{\Fp}/\delta}  
        \right)
        }
        }
        .
    \end{align*}
    \endgroup
    %
\end{theorem}

We remark that using covering numbers analysis~\citep{shalev2014UnderstandingMLBook}, our result naturally generalizes to infinite function classes as well as context spaces.
In addition, when comparing our regret upper bound to the lower bound of~\citet{levy2022optimism}, there is an apparent gap of $H^{2.5}$ and 
$\dEP$ factors. We leave this gap for future research.

\noindent\textbf{Computational efficiency of E-UC$^3$RL.}
{
The algorithm calls each oracle $T$ times, making it oracle-efficient (since it's oracle-call complexity is in $poly(T)$).
Aside from simple arithmetic operations, each of the $T(T+1)/2$ iterations of the internal loop call one MDP planning procedure and calculate the related occupancy measure. Both of these can be implemented efficiently using dynamic programming. 
%
%
Overall, excluding the oracle's computation time, the run-time complexity of our algorithm is in $\poly(T,|S|,|A|,H)$. 
Hence, if both the LSR and LLR oracles are computationally efficient then E-UC$^3$RL is also computationally efficient.
}


\section{Analysis} \label{sec:analysis}
Our analysis consists of four main steps:
\begin{enumerate}[nosep,label=(\roman*)]
    \item Establish an upper bound on the expected regret of the square and log loss regression oracles;
    \item Construct confidence bounds over the expected value of any context-dependent policy for both dynamics and rewards;
    \item Define the optimistic approximated CMDP and establish optimism lemmas;
    \item Combine the above to derive a high probability regret bound. 
\end{enumerate}

\noi
In what follows, we present the main claims of our analysis, deferring the proofs to \cref{sec:proofs}.
Before beginning, we discuss some of the challenges and present a key technical result, the value change of measure lemma (\cref{lemma:value-change-of-measure-main}).

\subsection*{A Key Technical Challenge}
Our goal is to derive computable and reliable confidence bounds over the expected value of any policy. The difficulty is that the offline regression oracles have regret guarantees only with respect to the trajectories' distributions, which are related to the true context-dependent dynamics $P_\star$. Hence, a main technical challenge is to translate the oracle's regret to a guarantee with respect to the estimated context-dependent dynamics $\widehat{P}_t$. Notice that the confidence bounds are computable only if stated in terms of $\widehat{P}_t$. 
Following ideas from \citet{foster2021statistical},
we solve this issue using a multiplicative value change of measure that is based on the Hellinger distance.
Concretely, the following change of measure lemma allows us to measure the value difference caused by the use of approximated transition probabilities in terms of the expected cumulative Hellinger distance (proof in \cref{sec:proof-change-of-measure}).
\begin{lemma}[Value change of measure]
\label{lemma:value-change-of-measure-main}
    Let $r: S \times A \to [0,1]$ be a bounded expected rewards function. Let $P_\star$ and $\widehat{P}$ denote two dynamics and consider the MDPs $M  = (S,A,P_\star,r, s_0,H)$ and $\widehat{M}  = (S,A,\widehat{P},r, s_0,H)$.
    Then, for any policy $\pi$ it holds that
    \begin{align*}
        &
        V^\pi_{\widehat{M}}(s)
        \le
        3 V^\pi_{{M}}(s) +
        \\
        &
        9 H^2
        \mathop{\E}_{P_\star, \pi}
        \brk[s]*{
        \sum_{h=0}^{H-1}
        D_H^2(\widehat{P}(\cdot|s_{h},a_{h}), {P}_\star(\cdot|s_{h},a_{h}))
        ~\bigg|~ s_{0} = s
        }
        .
    \end{align*}
\end{lemma}
Notice that this bound is loose when the reward function is not small. However, it is significantly tighter than standard results when the rewards are small. 
For instance, later in the analysis we consider the reward
$
r^c(s,a) = (\widehat{f}_t(c,s,a) - f_\star(c,s,a))^2
$
that is the squared reward approximation error. Letting $\widehat{\M}= (S, A, \widehat{P}, r^c, s_0, H)$ and $\M = (S, A, {P}_\star, r^c, s_0, H)$, \cref{lemma:value-change-of-measure-main} implies that the expected reward approximation error with respect to $\widehat{P}$ is at most a constant multiple of the expected reward and dynamics approximation errors with respect to $P_\star$. 
In contrast, a standard change of measure replaces the squared Hellinger distance with Total Variation (TV) whose cumulative error scales as $\smash{\sqrt{T}}$.


\subsection*{Step 1: Establishing Oracle Guarantees}

The regret guarantees of the least-squares oracle were established in \citet[Lemma B.10]{levy2022optimism}, stated in the Appendix as \cref{lemma:uniform-convergence}.
The following corollary bounds the cumulative expected least-squares loss of the sequence of the oracle's predictions (proof in \cref{sec:proofs-oracles}). 
In the following, we denote the expected squared error at round $t$ over a trajectory generated by $\pi$ where the context is $c$ as $\Esq[\pi]$. Formally,
\begin{align*}
    &
    \Esq[\pi] :=
    \\
    &\hspace{4mm}
    \mathop{\E}_{\pi(c;\cdot), P^{c}_\star} \Bigg[
    \sum_{h=0}^{H-1}
    \brk*{
    \widehat{f}_t(c, s_h, a_h) - f_\star(c, s_h, a_h)
    }^2
    ~\Bigg|~ s_0\Bigg]
    .
\end{align*}
\begin{corollary}[Reward approximation bound]\label{corl:reward-distance-bound-w.h.p-main}
    Let $\widehat{f}_t \in \F$ be the least squares minimizer of Line~\ref{ln:lsr-orace} in \cref{alg:UCCRL-main}.
    For any $\delta \in (0,1)$ it holds with probability at least $1-\delta$,
    \begin{align*}
        {\E}_c \brk[s]*{
        \sum_{i=1}^{t-1}
        \Esq
        }
        \leq
        68H \log(2 T^3 |\F|/\delta)
        ,
        \qquad
        \forall t \ge 1
        .
    \end{align*}
\end{corollary}

Next, we analyze the expected regret of the dynamic's log-loss oracle in terms of the Hellinger distance. 
%
The following result is a straightforward application of Lemma A.14 in~\citet{foster2021statistical} (proof in \cref{sec:proofs-oracles}).
%
%
%
%
%
%
To that end, we denote the expected squared Hellinger distance at round $t$ over a trajectory generated by $\pi$ where the context is $c$ as $\Ehelt[\pi]$. Formally,
\begin{align*}
   & \Ehelt[\pi]:=
    \\
    &\hspace{4mm}
    \mathop{\E}_{\pi(c;\cdot), P^{c}_\star} \Bigg[
        \sum_{h=0}^{H-1}D^2_H(P^{c}_\star(\cdot|s_{h},a_{h}),\widehat{P}_t^{c}(\cdot|s_{h},a_{h}))
        ~\Bigg|~ s_0\Bigg]
        .
\end{align*}
\begin{corollary}[Dynamics approximation bound]\label{corl:hellinger-distance-bound-w.h.p-main}
    Let ${\widehat{P}_t \in \Fp}$ be the log loss minimizer of Line~\ref{ln:lrr-oracle} in \cref{alg:UCCRL-main}.
    For any $\delta \in (0,1)$ it holds that with probability at least $1-\delta$,
    \begin{align*}
        {\E}_c \brk[s]*{
        \sum_{i=1}^{t-1}
        \Ehelt
        }
        \leq
        2H \log(TH  |\Fp|/\delta)
        ,
        \qquad
        \forall t \ge 1
        .
    \end{align*}
    
\end{corollary}

Lastly, we apply a variant of ~\cref{corl:hellinger-distance-bound-w.h.p-main} together with the shrinking confidence bound guarantee in \cref{lemma:russo-lemma2} to derive the following bound (proof in \cref{sec:proofs-oracles}).

\begin{lemma}[Stability error of log-loss oracle]\label{lemma:log-loss-regret}
    Let $\widehat{P}_i \in \Fp$ denote the log-loss minimizer at round $i \in [T]$.
    %
    For any $\delta \in (0,1)$ it holds with probability at least $1-\delta$ that
    \begin{align*}
        &\E_c \brk[s]*{ \sum_{i=1}^{t-1}
        \Eheli
        }
        \leq
        112 H 
        \dEP
        \log^2(2TH  |\Fp|/\delta)
        ,
        \;
        \forall t \ge 1
        ,
    \end{align*}
    where $\dEP \ge \dHE(\Fp, D_H, T^{-1/2})$, the Eluder dimension of $\Fp$ at scale $T^{-1/2}$.
\end{lemma}

One of the novel contributions of our work is the use of the Eluder dimension to bound the stability error due to the log-loss oracle. This allows us to choose stable bonuses that yield valid and computable confidence bounds.


\subsection*{Step 2: Constructing Confidence Bounds}
Our main goal in this subsection is to upper bound w.h.p the expected value difference between the true and the empirical CMDPs, for any context-dependent policy $\pi$ (\cref{corl:UCB-value-main}).
For that purpose, we derive confidence bounds over the rewards and dynamics approximation. 
Let $\pi_t$ denote the context-dependent policy selected at round $t$.
%
For any ${h \in [H]}$ and state-action pair $(s,a) \in S_h \times A$, context $c \in \C$, and round $t \ge 1$, we define the reward bonuses 
%
%
\begin{equation}
\label{eq:reward-bonuses}
\begin{aligned}
    b^R_{t,h}(c,s,a) 
    &: =
    b^{\betaR}_t(c,s,a,h)
    ,
     \\
    b^P_{t,h}(c,s,a) 
    &
    : = 
    Hb^{\betaP}_t(c,s,a,h)
    ,
 \end{aligned}
\end{equation}
where $b^\beta_t(c,s,a,h)$ is defined in Line~\ref{ln:bonuses}.
$b^R_{t,h}$ is the bonus related to the rewards approximation error, and $ b^P_{t,h}$ is the bonus related to that of the dynamics.
We remark that these bonuses differ only in constant terms ($H\betaP$ versus $\betaR$), and are identical  to the bonus terms defined in~\cref{alg:UCCRL-main} 
 (Line~\ref{ln:bonuses}).
We use these bonuses in our optimistic construction to account for the approximation errors in the rewards and dynamics, respectively.
Next, for any context $c \in \C$ and functions $f \in \F, P \in \Fp$ we define the MDP $\M^{(f,P)}(c) = (S,A,P^c,f(c,\cdot,\cdot),s_0,H)$.
The following results derive confidence bounds for the dynamics and rewards approximation in terms of the reward bonuses and approximation errors (proofs in \cref{sec:proofs-confidence-bounds}).



\begin{lemma}[Confidence bound for rewards approximation w.r.t the approximated dynamics]\label{lemma:CB-rewards-approx-p-main}
    Let $P_\star$ and $f_\star$ be the true context-dependent dynamics and rewards. Let $\widehat{P}_t$ and $\widehat{f}_t$ be the approximated context-dependent dynamics and rewards at round $t$.
    Then, 
    for any $t\geq 1$, and context-dependent policy $\pi \in \Pi_\C$, 
    \begingroup \allowdisplaybreaks
    \begin{align*}
        &\abs*{
        \E_c\brk[s]*{V^{\pi(c;\cdot)}_{\M^{(f_\star,\widehat{P}_t)}(c)}(s_0) }
        -
         \E_c\brk[s]*{V^{\pi(c;\cdot)}_{\M^{(\widehat{f}_t, \widehat{P}_t)}(c)}(s_0)} 
        }
        \leq
        \frac{H}{2 \betaR}
        \\
        &
        \hspace{1em}
        +
        \frac{3}{2\betaR}
        \E_c \brk[s]*{ \sum_{i=1}^{t-1} 
        \Esq
        }
        +
        \frac{9 H^2}{2\betaR} \E_c \brk[s]*{ \sum_{i=1}^{t-1} 
        \Eheli
        }
        \\
        &
        \hspace{1em}
        +
        \E_c\brk[s]*{
        \sum_{h=0}^{H-1}
        \sum_{s \in S^c_h} \sum_{a \in A}
        q_h(s,a|\pi(c;\cdot),\widehat{P}^c_t)
        \cdot b_{t,h}^R(c,s,a)
        }
        .
    \end{align*}
    \endgroup
\end{lemma}

\begin{lemma}[Confidence bound for dynamics approximation w.r.t the true rewards $f_\star$]\label{lemma:CB-dynamics-true-r-main} 
    Let $P_\star$ and $f_\star$ be the true context-dependent dynamics and rewards. Let $\widehat{P}_t$ be the approximated context-dependent dynamics at round $t$.
    Then, 
    for any $t\geq 1$, and context-dependent policy $\pi \in \Pi_\C$, 
    \begingroup
    \allowdisplaybreaks
    \begin{align*}  
        &
        \abs*{
        \E_c\brk[s]*{V^{\pi(c;\cdot)}_{\M^{({f}_\star,P_\star)}(c)}(s_0) 
        }
        -
        \E_c\brk[s]*{
        V^{\pi(c;\cdot)}_{\M^{({f}_\star, \widehat{P}_t)}(c)}(s_0)}
        }
        \leq 
        \frac{H^2}{2 \betaP}
        \\
        &\hspace{1em}
        +
        \E_c\brk[s]*{
        \sum_{h=0}^{H-1}
        \sum_{s \in S_h^c}
        \sum_{a \in A}
        q_h(s,a| \pi(c;\cdot), \widehat{P}^c_t )
        \cdot b^P_{t,h}(c,s,a)
        }
        \\
        &
        \hspace{1em}
        +
        \frac{6 H}{ \betaP}
        \E_c\brk[s]*{\sum_{i=1}^{t-1} 
        \Ehelt}
        +
        \frac{ 18 H^3}{\betaP}
       \E_c\brk[s]*{\sum_{i=1}^{t-1} 
        \Eheli
        }
        .
    \end{align*}
    \endgroup
\end{lemma}
The proofs of \cref{lemma:CB-rewards-approx-p-main,lemma:CB-dynamics-true-r-main} are similar and can be found in~\cref{sec:proofs-confidence-bounds}.
By applying the high probability approximation bounds in \cref{corl:reward-distance-bound-w.h.p-main,corl:hellinger-distance-bound-w.h.p-main,lemma:log-loss-regret} to \cref{lemma:CB-rewards-approx-p-main,lemma:CB-dynamics-true-r-main}, we obtain the desired high probability confidence bound on the expected value approximation (proof in \cref{sec:proofs-confidence-bounds}).
\begin{corollary}\label{corl:UCB-value-main}
    Under the terms of \cref{lemma:CB-rewards-approx-p-main,lemma:CB-dynamics-true-r-main}, the following holds with probability at least $1-3\delta/4$ simultaneously for all $t \geq 1$ and $\pi \in \Pi_\C$.
    \begingroup \allowdisplaybreaks
    \begin{align*}
        &
        \abs*{
        \E_c\brk[s]*{ V^{\pi(c;\cdot)}_{\M^{(f_\star,{P}_\star)}(c)}(s_0)
         }
         -
          \E_c\brk[s]*{
         V^{\pi(c;\cdot)}_{\M^{(\widehat{f}_t, \widehat{P}_t)}(c)}(s_0) }
        }
        \le
        \\
        &
        \E_c\big[
        \sum_{h=0}^{H-1}
        \sum_{s_h \in S^c_h} \sum_{a_h \in A}
        q_h(s_h,a_h|\pi(c;\cdot),\widehat{P}^c_t)
        b_{t,h}^R(c,s_h,a_h)
        \big] 
        \\
        &
        +
        \E_c\big[
        \sum_{h=0}^{H-1}
        \sum_{s_h \in S^c_h} \sum_{a_h \in A}
        q_h(s_h,a_h|\pi(c;\cdot),\widehat{P}^c_t)
        b_{t,h}^P(c,s_h,a_h)
        \big]
        \\
        &
        +
        \frac{2029 H^4 \dEP}{\betaP}
        \log^2(8TH  |\Fp|/\delta)
        \\
        &
        +
        \frac{504 H^3 \dEP }{\betaR} 
        \log^2(64T^4 H  |\F| |\Fp|/\delta^2)
        .
    \end{align*}
    \endgroup
\end{corollary}

\subsection*{Step 3: Establishing Optimism Lemmas}
In this subsection,
we use the results of step $2$ (\cref{corl:UCB-value-main}) to establish the properties of the optimistic approximated CMDP $\Mhat_t$. Namely, that its optimal value is higher than that of the true CMDP $\M$ (\cref{lemma:opt-in-expectation-main}), but also that the optimistic value of $\pi_t$ is not significantly higher than its true value (\cref{lemma:optimism-cost-main}), in expectation over the context. Combining both results and applying Azuma's inequality yields the regret bound.
%

We begin by defining the optimistic-in-expectation context-dependent reward function at every round $t \ge 1$ as 
$
    \widehat{r}^c_t (s,a) :=
    \widehat{f}_t (c,s,a) + b^R_{t,h}(c,s,a) +  b^P_{t,h}(c,s,a)
$    
, 
where the bonuses are defined in \cref{eq:reward-bonuses},
and we note that $\widehat{r}^c_t (s,a) \in [0,H+2]$ for all $h \in [H]$ and $(c,s,a) \in \C \times S_h \times A$.
The approximated optimistic-in-expectation CMDP at round $t$ is defined as $(\C,S,A,\Mhat_t)$
where for any context $c \in \C$ we define 
$
    \Mhat_t(c) := (S,A,\widehat{P}^c, \widehat{r}^c, s_0, H)
$.
We also recall that $\M(c) = \M^{(f_\star,P_\star)}(c)$ is the true CMDP.
%
%
%
%
%
%
%
%
%
The next two Lemmas establish the properties of the optimistic CMDP
(proofs in \cref{sec:proofs-optimism}).  
\begin{lemma}[Optimism in expectation]\label{lemma:opt-in-expectation-main}
    Let $\pi_\star$ be an optimal context-dependent policy for $\M$. 
    Under the good event of \cref{corl:UCB-value-main}, for any $t \geq 1$ it holds that
    \begingroup \allowdisplaybreaks
    \begin{align*}
        &\E_c\left[V^{\pi_\star(c;\cdot)}_{\M(c)}(s_0)
        \right]
        \leq
        \E_c \left[ 
        V^{\pi_t(c;\cdot)}_{\Mhat_t(c)}(s_0)\right]
        \\
        &
        +
        \frac{2029 H^4 \dEP}{\betaP}
        \log^2\brk*{{8TH  |\Fp|}/{\delta}}
        \\
        &
        +
        \frac{504 H^3 \dEP }{\betaR} 
        \log^2\brk*{{64T^4 H  |\F| |\Fp|}/{\delta^2}}
        .
    \end{align*}
    \endgroup
\end{lemma}

\begin{lemma}[The cost of approximation]\label{lemma:optimism-cost-main}
    Under the good event of \cref{corl:UCB-value-main}, we have that for every $t \geq 1$
    \begingroup\allowdisplaybreaks
    \begin{align*}
        &\E_c\left[ V^{\pi_t(c;\cdot)}_{\Mhat_t(c)}(s_0)\right]
        \leq
        \E_c\left[ V^{\pi_t(c;\cdot)}_{\M(c)}(s_0)\right]
        \\
        &
        +
        2 \sum_{h=0}^{H-1}\E_{c} \brk[s]*{
        \mathop{\E}_{\pi_t(c;\cdot), \widehat{P}_t^c} \brk[s]*{
        b^R_{t,h}(c,s_h,a_h) + b^R_{t,h}(c,s_h,a_h)
        }
        }
        \\
        &
        +
        \frac{2029 H^4 \dEP}{\betaP}
        \log^2(8TH  |\Fp|/\delta)
        \\
        &
        +
        \frac{504 H^3 \dEP }{\betaR} 
        \log^2(64T^4 H  |\F| |\Fp|/\delta^2)
        .
    \end{align*}
    \endgroup
\end{lemma}


\subsection*{Step 4: Deriving the Regret Bound}
Using the above results, we derive~\cref{thm:UC3RL-regret-bound} as follows.
Summing \cref{lemma:opt-in-expectation-main,lemma:optimism-cost-main} over $1 \le t \le T$ bounds a notion of expected regret.
Next, we use a standard algebraic argument (\cref{lemma:cummulative-bonuses}) to bound the expected cumulative bonuses as
\begin{align*}
    &\sum_{t=1}^{T}\sum_{h=0}^{H-1}\E_{c} \brk[s]*{
    \mathop{\E}_{\pi_t(c;\cdot), \widehat{P}_t^c} \brk[s]*{
    b^R_{t,h}(c,s_h,a_h) + b^R_{t,h}(c,s_h,a_h)
    }
    }
    \\
    & \le
    H|S||A| \brk{\betaR + H \betaP} \log(T+1)
    .
\end{align*}
%
By plugging in our choice of $\betaP$ and $\betaR$ we bound the expected regret by
\[
        271 H^3
        \sqrt{ T \abs{S} \abs{A}
        \dEP
        \log (T+1)
        \log^2 \brk{18 T^4 H \abs{\F}\abs{\Fp}/ \delta^2}
        }
        .
\]
We derive the high probability result by using \cref{corl:UCB-value-main} and applying Azuma's inequality.
The complete proof is in \cref{sec:regret-proof}.

\section{Discussion and Conclusion}
In this paper, we make a step forward in understanding RL with offline function approximation. We consider the tabular CMDP setting, under the offline function approximation assumption, and obtain a rate-optimal regret bound.
To obtain our result, we extend the Eluder dimension presented by~\citet{russo2013eluder} to a general bounded metric, rather than only the $\ell_2$ norm. This result may be of separate interest. Further, by applying the Metric Eluder dimension with Hellinger distance, we obtain 
our algorithm EUC$^3$RL, the first efficient algorithm for regret minimization in Contextual MDPs that uses offline regression oracles. We note that our algorithm requires a known bound on the Eluder dimension and its regret depends on it. 
Obtaining an efficient algorithm that has rate-optimal regret using offline oracles but without dependence on the Eluder dimension is an important open question for future research. Extending our technique to RL with rich observations is also an interesting direction for future research.

\section*{Acknowledgements}
We would like to thank the reviewers for their helpful comments.

This project has received funding from the European Research Council (ERC) under the European Union’s Horizon
2020 research and innovation program (grant agreement No.
882396 and grant agreement No. 101078075). Views and opinions expressed are
however those of the author(s) only and do not necessarily
reflect those of the European Union or the European Research Council. Neither the European Union nor the granting authority can be held responsible for them. 
This work received additional support from the Israel Science Foundation (ISF, grant numbers 993/17 and 2549/19), Tel Aviv University Center for AI
and Data Science (TAD), the Yandex Initiative for Machine
Learning at Tel Aviv University, the Len Blavatnik and the
Blavatnik Family Foundation, and by the Israeli VATAT data science scholarship.

AC is supported by the Israeli Science Foundation (ISF)
grant no. 2250/22.

\section*{Impact Statement}
This paper presents work whose goal is to advance the field of 
Machine Learning. There are many potential societal consequences 
of our work, none of which we feel must be specifically highlighted here.

\bibliography{references}
\bibliographystyle{icml2024}

\newpage
\appendix
\onecolumn




\section{Metric Eluder Dimension}
\label{sec:eluder-proofs}
Our goal is to prove \cref{lemma:russo-lemma2}, a variant of Proposition 6 in \citet{osband2014model}.
Recall that for any $\Fp' \subseteq \Fp$, its radius at $x \in \X$ is
\begin{align*}
    w_{\Fp'}(x)
    =
    \sup_{P,P' \in \Fp'} D({P}(x), {P'}(x))
    .
\end{align*}
Now, for any $t \in [T], h \in [H]$ let $x_h^t \in \X$ and ${P}_t \in \Fp$ be arbitrary.
    Define the confidence sets with parameter $\beta$ as
\[
    \Fp_t
    =
    \brk[c]*{
    P\in \Fp
    \;:\;
    \sum_{i=0}^{t-1} \sum_{h=0}^{H-1} D^2(P(x_h^i), {P}_t(x_h^i)
    )
    \le
    \beta
    }
    .
\]
\begin{lemma}[Restatement of \cref{lemma:russo-lemma2}]
    Suppose that $\beta \ge H$. We have that
    \begin{align*}
        \sum_{t=1}^{T}\sum_{h=0}^{H-1} \brk{w_{\Fp_t}(x^t_h)}^2
        \le
        6 \dHE(\Fp, D, T^{-1/2}) \beta \log T
        .
    \end{align*}
\end{lemma}

We first need the following result, which is a direct adaptation of Lemma 1 in \cite{osband2014model} (see proof below for completeness).
\begin{lemma}
\label{lemma:russo-prop3}
We have that for any $\epsilon > 0$
\begin{align*}
    \sum_{t=1}^{T}\sum_{h=0}^{H-1}
    \I\brk[s]{w_{\Fp_t}(x_h^t) > \epsilon}
    \le
    \brk*{\frac{4 \beta}{\epsilon^2} + H}\dHE(\Fp, D, \epsilon)
    .
\end{align*}
\end{lemma}
\begin{proof}[Proof of \cref{lemma:russo-lemma2}]
    To reduce notation, write $w_{t,h} = w_{\Fp_t}(x^t_h)$, and $d = \dHE(\Fp, D, T^{-1/2})$.
    Next, reorder the sequence $(w_{1,0}, \ldots, w_{1,H-1}, \ldots, w_{T,0}, \ldots, w_{T,H-1}) \to (w_{i_1}, \ldots, w_{i_{TH}})$ where $w_{i_1} \ge w_{i_2} \ge \ldots \ge w_{i_{TH}}$.
    Then we have that
    \begin{align*}
        \sum_{t=1}^{T} \sum_{h=0}^{H-1} \brk{w_{\Fp_t}(x^t_h)}^2
        =
        \sum_{s=1}^{TH} w_{i_s}^2
        &
        =
        \sum_{s=1}^{TH} w_{i_s}^2 \I\brk[s]{w_{i_s} \le T^{-1/2}}
        +
        \sum_{s=1}^{TH} w_{i_s}^2 \I\brk[s]{w_{i_s} > T^{-1/2}}
        \\
        &
        \le
        \sum_{s=1}^{TH} \frac{1}{T}
        +
        \sum_{s=1}^{TH} w_{i_s}^2 \I\brk[s]{w_{i_s} > T^{-1/2}}
        \\
        &
        =
        H
        +
        \sum_{s=1}^{T} w_{i_s}^2 \I\brk[s]{w_{i_s} > T^{-1/2}}
        .
    \end{align*}
    Now, let $\epsilon \ge T^{-1/2}$ and suppose that $w_{i_s} > \epsilon$. Since $w_{i_s}$ are ordered in descending order, this implies that $\sum_{t=1}^{T}\sum_{h=0}^{H-1} \I\brk{w_{\Fp_t}(x_h^t) > \epsilon} \ge s$. On the other hand, by \cref{lemma:russo-prop3}, we have
    \begin{align*}
        \sum_{t=1}^{T}\sum_{h=0}^{H-1} \I\brk{w_{\Fp_t}(x_h^t) > \epsilon}
        \le
        \brk*{\frac{4 \beta}{\epsilon^2} + H}\dHE(\Fp, D, \epsilon)
        \le
        \brk*{\frac{4 \beta}{\epsilon^2} + H} d
        ,
    \end{align*}
    where the second transition used that $\dHE(\Fp,D, \epsilon')$ is non-increasing in $\epsilon'$. We conclude that 
    $
    s
    \le
    \brk*{{4 \beta}/{\epsilon^2} + H}d
    ,
    $
    and changing sides gives that
    $
    \epsilon^2
    \le
    4 \beta / (s - d H)
    .
    $
    This implies that
    \begin{align*}
        w_{i_s} > T^{-1/2}
        \implies
        w_{i_s}^2
        \le
        \frac{4 \beta d}{s - d H}
        .
    \end{align*}
    We thus have
    \begin{align*}
        \tag{$w_{i_s} \le 1$}
        \sum_{s=1}^{TH} w_{i_s}^2 \I\brk{w_{i_s} > T^{-1/2}}
        &
        \le
        1 + d H
        +
        \sum_{s=d H+2}^{T} w_{i_s}^2 \I\brk{w_{i_s} > T^{-1/2}}
        \\
        &
        \le
        1 + d H
        +
        \sum_{s=d H+2}^{T} \frac{4\beta d}{s-dH}
        \\
        &
        \le
        1 + d H
        +
        4 \beta d \brk*{\int_{t=1}^{T} \frac{1}{t} dt}
        \\
        &
        =
        1 + d H + 4 \beta d \log T
        \\
        \tag{$\beta \ge H$}
        &
        \le
        6 d \beta \log T
        .
    \end{align*}
\end{proof}

\begin{proof}[Proof of \cref{lemma:russo-prop3}]
    The proof follows similarly to Lemma 1 in \citet{osband2014model}.
    We begin by showing that if $w_{\Fp_t}(x^t_h) > \epsilon$ then $x^t_h$ is $(D,\epsilon)-$dependent on fewer than $4 \beta/ \epsilon^2$ disjoint sub-sequences of 
    $
    (x^1_0, \ldots, x^1_{H-1}, \ldots, x^{t-1}_0, \ldots, x^{t-1}_{H-1})
    .
    $
    To see this, note that if $w_{\Fp_t}(x^t_h) > \epsilon$ there are $P, P' \in \Fp_t$ such that
    \begin{align*}
        D(P(x^t_h), P'(x^t_h))
        >
        \epsilon
        .
    \end{align*}
    Now, suppose that $x^t_h$ is $(D,\epsilon)-$dependent on a sub-sequence $(x^{t_1}_{h_1}, \ldots, x^{t_n}_{h_n})$ of 
    $
    (x^1_0, \ldots, x^1_{H-1}, \ldots\linebreak[1], x^{t-1}_0, \ldots, x^{t-1}_{H-1})
    .
    $
    Since 
    $
    D(P(x^t_h), P'(x^t_h))
    \ge
    \epsilon
    ,
    $
    the dependence implies that
    \begin{align*}
        \sum_{j=1}^{n} D^2(P(x^{t_j}_{h_j}), P'(x^{t_j}_{h_j}))
        >
        \epsilon^2
        .
    \end{align*}
    We conclude that, if $x^t_h$ is $\epsilon-$dependent on $K$ disjoint sub-sequences of 
    $
    (x^1_0, \ldots, x^1_{H-1}, \ldots, x^{t-1}_0\linebreak[1], \ldots, x^{t-1}_{H-1})
    $
    then
    \begin{align*}
        \sum_{i=1}^{t-1}\sum_{h=0}^{H-1} D^2(P(x^{i}_h), P'(x^{i}_h))
        >
        K \epsilon^2
        .
    \end{align*}
    On the other hand, since $P,P' \in \Fp_t$, we have
    \begin{align*}
        \sum_{i=1}^{t-1}\sum_{h=0}^{H-1} D^2(P(x^{i}_h), P'(x^{i}_h))
        &
        \tag{triangle inequality for $D$}
        \le
        \sum_{i=1}^{t-1}\sum_{h=0}^{H-1} \brk[s]*{
        D(P(x^{i}_h), {P}_t(x^{i}_h))
        +
        D({P}_t(x^{i}_h), P'(x^{i}_h))
        }^2
        \\
        \tag{$(a+b)^2 \le 2(a^2+b^2)$}
        &
        \le
        2\sum_{i=1}^{t-1}\sum_{h=0}^{H-1} D^2(P(x^{i}_h), {P}_t(x^{i}_h))
        +
        2\sum_{i=1}^{t-1}\sum_{h=0}^{H-1} D^2({P}_t(x^{i}_h), P'(x^{i}_h))
        \\
        &
        \le
        4 \beta
        .
    \end{align*}
    It follows that $K < 4 \beta / \epsilon^2$.

    Next, denote $d := \dHE(\Fp, D, \epsilon)$. We show that in any action sequence $(y_1, \ldots, y_\tau)$, there is some element $y_j$ that is $(D,\epsilon)-$dependent on at least $\tau/d - 1$ disjoint sub-sequences of $(y_1, \ldots, y_{j-1})$. To show this, we will construct $K = \ceil{(\tau/d) - 1}$ disjoint sub-sequences $B_1, \ldots, B_K$. First, let $B_m = (y_m)$ for $m = 1,\ldots, K$. If $y_{K+1}$ is $(D,\epsilon)-$dependent on each sub-sequence $B_1, \ldots, B_K$, the claim is established. 
    Otherwise, append $y_{K+1}$ to a sub-sequence $B_m$ that $y_{K+1}$ is $(D,\epsilon)-$independent of. 
    Repeat this process for elements with indices $i > K+1$ until $y_i$ is $\epsilon-$dependent on all sub-sequences. 
    Suppose in contradiction that the process terminated without finding the desired $y_i$. 
    By the definition of the Eluder dimension, each $B_m$ must satisfy $\abs{B_m} \le d$ and thus $\sum_{m=1}^{K} \abs{B_m} \le K d$. On the other hand
    \begin{align*}
        K d
        =
        \ceil{(\tau / d) - 1} d
        <
        \tau
        =
        \sum_{m=1}^{K} \abs{B_m}
        \le
        K d
        ,
    \end{align*}
    where the last equality follows since all elements of $(y_1, \ldots , y_\tau)$ where placed. This is a contradiction, and thus the process must terminate successfully.
    %
    
    Now, let 
    \[
    \tau
    =
    \sum_{t=1}^{T}\sum_{h=0}^{H-1}
    \I\brk[s]{w_{\Fp_t}(x_h^t) > \epsilon}
    ,
    \]
    and take $(y_1, \ldots, y_\tau)$ to be the sub-sequence 
    $
    (x^{t_1}_{h_1}, \ldots, x^{t_\tau}_{h_\tau})
    $
    of elements $x^t_h$ for which $w_{\Fp_t}(x^t_h) > \epsilon$. Then there exists $x^{t_j}_{h_j}$ that depends on at least $\tau / d - 1$ disjoint sub-sequences
    $
    (x^{t_1}_{h_1}, \ldots, x^{t_\tau}_{h_\tau})
    .
    $
    Notice that at most $H-1$ of these sub-sequences are not sub-sequences of 
    $
    (x^1_0, \ldots, x^1_{H-1}, \ldots, x^{t_j-1}_0, \ldots\linebreak[1], x^{t_j-1}_{H-1})
    .
    $
    We conclude that $x^{t_j}_{h_j}$ depends on at least $(\tau / d) - H$ disjoint sub-sequences of
    $
    (x^1_0, \ldots \linebreak[1], x^1_{H-1}, \ldots, x^{t_j-1}_0, \ldots x^{t_j-1}_{H-1})
    .
    $
    On the other hand, since $w_{\Fp_{t_j}}(x^{t_j}_{h_j}) > \epsilon$ then by the first part of the proof it depends on fewer than $4 \beta / \epsilon^2$ such disjoint sub-sequences, thus
    \[
    \frac{\tau}{d} - H \le \frac{4 \beta}{\epsilon^2}
    ,
    \]
    so $\tau \le (4 \beta / \epsilon^2 + H) d$ as desired.
    %
    %
\end{proof}

%
%
%

\subsection{Upper bound using the $\ell_2$ Eluder dimension}

\begin{proposition}\label{prop:hel}
    Assume that $P,Q$ are two distributions over finite domain $\mathcal{X}$, such that for all $x \in \mathcal{X}$ it holds that $P(x),Q(x) \geq p$. Then,
    \begin{align*}
        \norm{P-Q}^2_2 
        \leq 
        4 D^2_H(P,Q)
        \leq
        \frac{1}{p}\norm{P-Q}^2_2 
    \end{align*}
\end{proposition}

\begin{lemma}\label{lemma:Eluder-upper-bound-p-min}
    Let function class $\Fp$, and assume $P(x) \geq p$ holds for some $p >0$ for every $P \in \Fp$ and $x$ in the domain. 
    Let $\epsilon \in (0,1]$ 
     and assume the the Eluder dimension w.r.t the $\ell_2$ at scale $\sqrt{p}\epsilon$ is $d$. Then, the Eluder dimension w.r.t the Squared Hellinger Distance at scale $\epsilon$ is upper bounded by $d/p$.
\end{lemma}

\begin{proof}
    Let $d$ be the Eluder dimension w.r.t the $\ell_2$-norm at scale $\sqrt{p} \epsilon$ of the class $\Fp$.
    Let $k > \ceil{1/p}$ and $m = d k$. Also let $x_1, x_2, \ldots, x_{m+1}$ be an arbitrary sequence. We show that there must exist $j \in [m+1]$ such that $x_j$ is $\epsilon$ dependent on its prefix in Hellinger distance, thus the Hellinger Eluder dimension at scale $\epsilon$ is at most $m$.
    As in~\cref{lemma:russo-prop3}, there exists $n$ and $\brk[c]{I_j}_{j \in [k]}$ disjoint subsequences such that $\cup_{j \in [k]} I_j = [n-1]$, and $x_n$ is $\sqrt{p} \epsilon$ dependent in $\ell_2$ on each $I_j, j \in [k]$.

    Now, let $P,P' \in \Fp$ such that $D_H^2(P(x_n), P'(x_n)) \ge \epsilon^2$. Then, by~\cref{prop:hel} we have that
    $\norm{P(x_n) - P'(x_n)}_2^2 \ge 4 p \epsilon^2$. Because of the dependence, for all $j \in [k]$ we have that $\sum_{\tau \in I_j} \norm{P(x_\tau) - P'(x_\tau)}_2^2 \ge 4 p \epsilon^2$. We conclude, using~\cref{prop:hel} again, that
    \begin{align*}
        \sum_{\tau = 1}^{n-1} D_H^2(P(x_\tau), P'(x_\tau))
        \ge
        \frac{1}{4}
        \sum_{\tau = 1}^{n-1} \norm{P(x_\tau) - P'(x_\tau)}_2^2
        =
        \frac{1}{4}
        \sum_{j \in [k]} \sum_{\tau \in I_j}
        \norm{P(x_\tau) - P'(x_\tau)}_2^2
        \ge
        \sum_{j \in [k]} p \epsilon^2
        =
        p k \epsilon^2
        >
        \epsilon^2
        ,
    \end{align*}
    thus $x_n$ is $\epsilon$ dependent on its prefix in Hellinger distance. We conclude that no sequence of length $m+1$ can have all elements $\epsilon$ independent of their prefix in Hellinger distance. Thus, the Hellinger Eluder at scale $\epsilon$ is at most $m \approx d / p.$

\end{proof}

\section{Proofs}
\label{sec:proofs}

\subsection{Multiplicative Value Change of Measure}
\label{sec:proof-change-of-measure}

First, we give the following Bernstein type tail bound \citep[see e.g.,][Lemma D.4]{rosenberg2020near}.
\begin{lemma}
\label{lemma:multiplicative-concentration}
Let $\brk[c]{X_t}_{t \ge 1}$
be a sequence of random variables with expectation adapted to a filtration
$\mathcal{F}_t$.
Suppose that $0 \le X_t \le 1$ almost surely. Then with probability at least $1-\delta$
\begin{align*}
    \sum_{t=1}^{T} \E \brk[s]{X_t \mid \mathcal{F}_{t-1}}
    \le
    2 \sum_{t=1}^{T} X_t
    +
    4 \log \frac{2}{\delta}
\end{align*}
\end{lemma}

Recall the Helligner distance given in \cref{def:hellinger}. The following change of measure result is due to \cite{foster2021statistical}.
\begin{lemma}[Lemma A.11 in~\citealp{foster2021statistical}]\label{lemma:A.14-dylan}
Let $\mathbb{P}$ and $\mathbb{Q}$ be two probability measures on $(\mathcal{X}, \mathrm{F})$. For all $h:\mathcal{X} \to \R $ with $0 \leq h(X) \leq R$ almost surely under $\mathbb{P}$ and $\mathbb{Q}$, we have
\begin{align*}
    \left| \E_{\mathbb{P}}[h(X)] -  \E_{\mathbb{Q}}[h(X)] \right| 
    \leq
    \sqrt{2 R ( \E_{\mathbb{P}}[h(X)] +  \E_{\mathbb{Q}}[h(X)])\cdot D^2_H(\mathbb{P}, \mathbb{Q})}
    .
\end{align*}
In particular,
\begin{align*}
    \E_{\mathbb{P}}[h(X)] 
    \leq
    3\E_{\mathbb{Q}}[h(X)] + 4RD^2_H(\mathbb{P}, \mathbb{Q})
    .
\end{align*}
\end{lemma}

Next, we need the following refinement of the previous result.
\begin{corollary}\label{corl:A.11}
    For any $\beta \geq 1$,
    \begin{align*}
    \E_{\mathbb{P}}[h(X)] 
    \leq
    (1+1/\beta)\E_{\mathbb{Q}}[h(X)] + 3\beta RD^2_H(\mathbb{P}, \mathbb{Q})
    .
\end{align*}
\end{corollary}

\begin{proof}
    Let $\eta \in (0,1)$. Consider the following derivation.
    \begingroup
    \allowdisplaybreaks
    \begin{align*}
        \E_{\mathbb{P}}[h(X)] - \E_{\mathbb{Q}}[h(X)]
        &
        \leq
        \sqrt{2 R ( \E_{\mathbb{P}}[h(X)] +  \E_{\mathbb{Q}}[h(X)])\cdot D^2_H(\mathbb{P}, \mathbb{Q})}
        \\
        &
        \leq
        \eta (\E_{\mathbb{P}}[h(X)] + \E_{\mathbb{Q}}[h(X)]) + \frac{R}{2\eta}D^2_H(\mathbb{P}, \mathbb{Q}).
    \end{align*}
    \endgroup
    The above implies
    \begingroup\allowdisplaybreaks
    \begin{align*}
        \E_{\mathbb{P}}[h(X)] 
        &
        \leq 
        \frac{1+\eta}{1-\eta}\E_{\mathbb{Q}}[h(X)] + \frac{R}{2\eta (1-\eta)}D^2_H(\mathbb{P},\mathbb{Q})
        \\
        \tag{Plug $\eta = \frac{1}{2\beta+1}$ for all $\beta \in (0, \infty)$.}
        &
        =
        \left(1 + \frac{1}{\beta}\right)\E_{\mathbb{Q}}[h(X)] + 3R \frac{(2\beta +1)^2}{2\beta} D^2_H(\mathbb{P},\mathbb{Q})
        \\
        \tag{For any $\beta \geq 1$}
        &
        \leq
        \left(1 + \frac{1}{\beta}\right)\E_{\mathbb{Q}}[h(X)] + 3R \beta D^2_H(\mathbb{P},\mathbb{Q}).
    \end{align*}
    \endgroup
\end{proof}

\begin{lemma}[restatement of \cref{lemma:value-change-of-measure-main}]
    Let $r: S \times A \to [0,1]$ be a bounded expected rewards function. Let $P_\star$ and $\widehat{P}$ denote two dynamics and consider the MDPs $M  = (S,A,P_\star,r, s_0,H)$ and $\widehat{M}  = (S,A,\widehat{P},r, s_0,H)$.
    Then, for any policy $\pi$ we have
    \begin{align*}
        V^\pi_{\widehat{M}}(s)
        \le
        3 V^\pi_{{M}}(s)
        +
        9 H^2
        \mathop{\E}_{P_\star, \pi}
        \brk[s]*{
        \sum_{h=0}^{H-1}
        D_H^2(\widehat{P}(\cdot|s_{h},a_{h}), {P}_\star(\cdot|s_{h},a_{h}))
        \Bigg| s_{0} = s
        }
        .
    \end{align*}
\end{lemma}
\begin{proof}
    We first prove by backwards induction that for all $h \in [H-1]$ the following holds.
    \begin{align*}
        V^\pi_{\widehat{M},h}(s)
        \le
        \brk*{1+\frac1H}^{H-h}
        \brk[s]*{
        V^\pi_{{M},h}(s)
        +
        \mathop{\E}_{P_\star, \pi}
        \brk[s]*{
        \sum_{h'=h}^{H-1}
        3H^2 D_H^2(\widehat{P}(\cdot|s_{h'},a_{h'}), {P}_\star(\cdot|s_{h'},a_{h'}))
        \Bigg| s_{h} = s
        }
        }
        .
    \end{align*}
    The base case, $h=H-1$ is immediate since $V^\pi_{\widehat{M},h}(s) = V^\pi_{{M},h}(s)$.
    Now, we assume that the above holds for $h+1$ and prove that it holds for $h$.
    To see this, we have that
    \begingroup\allowdisplaybreaks
    \begin{align*}
        &
        V^\pi_{\widehat{M},h}(s)
        \tag{By Bellman's equations}
        = 
        \mathop{\E}_{a \sim \pi(\cdot|s)}
        \brk[s]*{
        r(s,a)
        +
        \E_{s' \sim \widehat{P}(\cdot|s,a)}
        \brk[s]*{V^\pi_{\widehat{M},h+1}(s') 
        } }
        \\
        \tag{\cref{corl:A.11}}
        \le &
        \mathop{\E}_{a \sim \pi(\cdot|s)}
        \brk[s]*{
        r(s,a)
        +
        \brk*{1+\frac1H}\E_{s' \sim {P}_\star(\cdot|s,a)}
        \brk[s]*{V^\pi_{\widehat{M},h+1}(s')
        }
        +
        3H^2 D_H^2(\widehat{P}(\cdot|s,a), {P}_\star(\cdot|s,a))
        } 
        \\
        \tag{Induction hypothesis}
        \le &
        \mathop{\E}_{a \sim \pi(\cdot|s)}
        \brk[s]*{
        r(s,a)
        +
        3H^2 D_H^2(\widehat{P}(\cdot|s,a), {P}_\star(\cdot|s,a))
        }
        \\
        + &
        \mathop{\E}_{a \sim \pi(\cdot|s)}
        \brk[s]*{
        \brk*{1+\frac1H}^{H-h}\mathop{\E}_{s' \sim {P}_\star(\cdot|s,a)}
        \brk[s]*{
        V^\pi_{{M},h+1}(s')
        }
        }
        \\
        + &
        \mathop{\E}_{a \sim \pi(\cdot|s)}
        \brk[s]*{
        \brk*{1+\frac1H}^{H-h}\mathop{\E}_{s' \sim {P}_\star(\cdot|s,a)}
        \brk[s]*{
        \E\brk[s]*{
        \sum_{h'=h+1}^{H-1}
        3H^2 D_H^2(\widehat{P}(\cdot|s_{h'},a_{h'}), {P}_\star(\cdot|s_{h'},a_{h'}))
        \Bigg| s_{h+1} = s'
        }
        }
        }
        \\
        \tag{$r, D_H^2 \ge 0$}
        \le &
        \brk*{1+\frac1H}^{H-h}
        \mathop{\E}_{a \sim \pi(\cdot|s)}
        \brk[s]*{
        r(s,a)
        +
        \mathop{\E}_{s' \sim {P}_\star(\cdot|s,a)}
        \brk[s]*{
        V^\pi_{{M},h+1}(s')
        }
        }
        \\
        + &
        \brk*{1+\frac1H}^{H-h}
        \mathop{\E}_{P_\star, \pi}
        \brk[s]*{
        \sum_{h'=h}^{H-1}
        3H^2 D_H^2(\widehat{P}(\cdot|s_{h'},a_{h'}), {P}_\star(\cdot|s_{h'},a_{h'}))
        \Bigg| s_{h} = s
        }
        \\
        \tag{By Bellman's equations}
        = &
        \brk*{1+\frac1H}^{H-h}
        \brk[s]*{
        V^\pi_{{M},h}(s)
        +
        \mathop{\E}_{P_\star, \pi}
        \brk[s]*{
        \sum_{h'=h}^{H-1}
        3H^2 D_H^2(\widehat{P}(\cdot|s_{h'},a_{h'}), {P}_\star(\cdot|s_{h'},a_{h'}))
        \Bigg| s_{h} = s
        }
        }
        ,
    \end{align*}
    \endgroup
    as desired. Plugging in $h=0$ and using that $\brk*{1+\frac1H}^{H} \le 3$ concludes the proof.
\end{proof}


\subsection{Oracle Bounds (Step 1)}
\label{sec:proofs-oracles}

\paragraph{Reward oracle.}
\begin{lemma}[Lemma B.10 in \citealp{levy2022optimism}]\label{lemma:uniform-convergence}
    For any $\delta \in (0,1)$, with probability at least $1-\delta$ we have
    \begin{align*}
        &\sum_{i=1}^{t-1}\E \brk[s]4{ \sum_{h=0}^{H-1} (f_t(c_i,s^i_h,a^i_h) - f_\star(c_i,s^i_h,a^i_h))^2 ~\bigg|~ \Hist_{i-1} }
        \\
        & =
        \sum_{i=1}^{t-1}\sum_{h=0}^{H-1} \E_{c_i,s^i_h,a^i_h} \brk[1]{ (f_t(c_i,s^i_h,a^i_h)-f_\star(c_i,s^i_h,a^i_h))^2 \mid \Hist_{i-1}}
        \\
        & \leq
        68H\log(2|\F|t^3/\delta) + 2\sum_{i=1}^{t-1}\sum_{h=0}^{H-1} (f_t(c_i,s^i_h,a^i_h) - r^i_h)^2  - (f_\star(c_i,s^i_h,a^i_h) - r^i_h)^2
    \end{align*}
    simultaneously, for all $t\geq 2$ and any fixed sequence of functions $f_1,f_2,\ldots \in \F$.
\end{lemma}

\begin{corollary}[restatement of \cref{corl:reward-distance-bound-w.h.p-main}]
    Let $\widehat{f}_t \in \F$ be the least squares minimizer in \cref{alg:UCCRL-main}.
    For any $\delta \in (0,1)$ it holds that with probability at least $1-\delta$ we have
    \begin{align*}
        &{\E}_c \brk[s]*{
        \sum_{i=1}^{t-1}
        \mathop{\E}_{\pi_i(c;\cdot), P^{c}_\star} \Bigg[
        \sum_{h=0}^{H-1}
        \brk*{
        \widehat{f}_t(c, s_h, a_h) - f_\star(c, s_h, a_h)
        }^2
        \Bigg| s_0\Bigg]}
        \leq
        68H \log(2 T^3 |\F|/\delta)
        ,
    \end{align*}
    simultaneously, for all $t \geq 1$. 
\end{corollary}

\begin{proof}
    Recall that for all $t\geq 2$, $\widehat{f}_t$ is the least square minimizer at round $t$.
    Hence, by our assumption that $f_\star \in \F$
    \begin{align*}
        \sum_{i=1}^{t-1} \sum_{h=0}^{H-1}(\widehat{f}_t(c_i,s^i_h,a^i_h)-r^i_h)^2 - (f_\star(c_i,s^i_h,a^i_h)-r^i_h)^2 \leq 0.
    \end{align*}
    Thus the corollary immediately follows by \cref{lemma:uniform-convergence}.
\end{proof}

\paragraph{Dynamics oracle.}

Recall the Hellinger distance given in \cref{def:hellinger}.
The following lemma by \cite{foster2021statistical} upper bounds the expected cumulative Hellinger Distance in terms of the log-loss.
Let $\X$ be a set and $\Y$ be a finite set. Let $x^{(t)}, y^{(t)}$, $t \ge 1$, be a sequence of random variables that satisfy $y^{(t)} \sim g_\star(\cdot | x^{(t)})$, where $g_\star : \Y \times \X \to \R_+$ maps $x \in \X$ to the density of $ y \in \Y$. Define $\mathcal{H}^{(t)} = (x^{(1)}, y^{(1)}, \ldots, x^{(t)}, y^{(t}))$ and let $\mathcal{G}^{(t)} = \sigma(\mathcal{H}^{(t)})$. Next, for a random variable $Z$, we define $\E_t Z := \E [Z | \mathcal{G}_t]$.
\begin{lemma}[Lemma A.14 from~\citealp{foster2021statistical}]\label{lemma:A.14-dylan-LL}
    Let $g: \Y \times \X \to \R_+$ be a mapping from $\X$ to densities over $\Y$.
    Consider a sequence of $\{0,1\}$-valued random variables $(\I_t)_{t \leq T}$ where $\I_t$ is $\mathit{F}^{(t-1)}$-measurable. For any $\delta \in (0,1)$ we have that with probability at least $1-\delta$,
    \begin{align*}
        \sum_{t=1}^T \E_{t-1}
        &
        \left[ D^2_H(g(\cdot | x^{(t)}), {g}_\star(\cdot | x^{(t)})) \right]\I_t
        \\
        &
        \leq
         \sum_{t=1}^T \left( \log\frac{1}{g(y^{(t)} | x^{(t)})} - \log\frac{1}{g_\star(y^{(t))} | x^{(t)}} \right)\I_t + 2 \log(1/\delta).
    \end{align*}
    Additionally, with probability at least $1-\delta$,
    \begin{align*}
        \sum_{t=1}^T  D^2_H(g(\cdot | x^{(t)}), {g}_\star(\cdot | x^{(t)}))
        \leq
         \sum_{t=1}^T \left( \log\frac{1}{g(y^{(t)} | x^{(t)})} - \log\frac{1}{g_\star(y^{(t))} | x^{(t)}} \right)
         +
         2 \log(1/\delta).
    \end{align*}
\end{lemma}

Using the above lemma, we bound the realized and expected cumulative Hellinger distance between the approximated and true dynamics, by the actual regret of the log-loss regression oracle (and constant terms), with high probability.
\begin{lemma}[Concentration of log-loss oracle]\label{lemma:transition-to-oracle-rgret-dynamics}
    For any $\delta \in (0,1)$ it holds that with probability at least $1-\delta$ we have
    \begin{align*}
        &{\E}_c \brk[s]*{
        \sum_{i=1}^{t-1}
        \mathop{\E}_{\pi_i(c;\cdot), P^{c}_\star} \Bigg[
        \sum_{h=0}^{H-1}D^2_H(P^{c}_\star(\cdot|s_{h},a_{h}),P^{c}(\cdot|s_{h},a_{h}))
        \Bigg| s_0\Bigg]}
        \\
        \leq & 
        \sum_{i=1}^{t-1} \sum_{h=0}^{H-1} \log\left( \frac{1}{P^{c_i}(s^i_{h+1}|s^i_h,a^i_h)}\right) 
        -
        \sum_{i=1}^{t-1} \sum_{h=0}^{H-1} \log\left( \frac{1}{P_\star^{c_i}(s^i_{h+1}|s^i_h,a^i_h)}\right) 
        + 2H \log(TH |\Fp|/\delta).
    \end{align*}
    simultaneously, for all $t \geq 1$ and $P \in \Fp$. 
\end{lemma}

\begin{proof}
    Fix some $t \ge 1$ and $P \in \Fp$. We have with probability at least $1-\frac{\delta}{T|\Fp|}$ that
    \begingroup
    \allowdisplaybreaks
    \begin{align*}
        &
        {\E}_c \brk[s]*{
        \sum_{i=1}^{t-1}
        \mathop{\E}_{\pi_i(c;\cdot), P^{c}_\star} \Bigg[
        \sum_{h=0}^{H-1}D^2_H(P^{c}_\star(\cdot|s_{h},a_{h}),P^{c}(\cdot|s_{h},a_{h}))
        \Bigg| s_0\Bigg]}
        \\
        &
        =
        \sum_{i=1}^{t-1}
        \sum_{h=0}^{H-1}
        {\E}_c\left[\mathop{\E}_{\pi_i(c;\cdot), P^{c}_\star} \Bigg[
        D^2_H(P^{c}_\star(\cdot|s_{h},a_{h}),P^{c}(\cdot|s_{h},a_{h}))
        \Bigg| s_0\Bigg]\right]
        \\
        &
        \underbrace{=}_{(i)}
        \sum_{h=0}^{H-1}
        \sum_{i=1}^{t-1}
        \mathop{\E} \Bigg[
        D^2_H(P^{c_i}_\star(\cdot|s^i_{h},a^i_{h}),P^{c_i}(\cdot|s^i_{h},a^i_{h}))
        \Bigg| s_0, \pi_i \Bigg]
        \\
        \tag{\citealp[Lemma A.14]{foster2021statistical}}
        & \leq
        \sum_{h=0}^{H-1}\left(\sum_{i=1}^{t-1}\log\left(\frac{P^{c_i}_\star(s^i_{h+1}|s^i_h,a^i_h)}{{P}^{c_i}(s^i_{h+1}|s^i_h,a^i_h)}\right) + 2 \log(HT |\Fp|/\delta)\right)
        \\
        & =
        \sum_{i=1}^{t-1} \sum_{h=0}^{H-1} \log\left(\frac{1}{{P}^{c_i}(s^i_{h+1}|s^i_h,a^i_h)} \right) -  \sum_{i=1}^{t-1} \sum_{h=0}^{H-1}\log\left(\frac{1}{P^{c_i}_\star(s^i_{h+1}|s^i_h,a^i_h)}\right) + 2 H\log(HT |\Fp|/\delta)
    \end{align*}
    \endgroup
    The filtration used in $(i)$ is over the history up to time $t$, $\Hist_{t-1} = (\sigma^1, \ldots,\sigma^{t-1})$.
    Now, by taking a union bound over every $t = 1.\ldots, T$ and $P \in \Fp$, we obtain the lemma.
\end{proof}

\begin{lemma}[Realized log-loss error]\label{lemma:realized-log-loss-regret}
    For any $\delta \in (0,1)$ it holds that with probability at least $1-\delta$ we have
    \begin{align*}
        &
        \sum_{i=1}^{t-1}
        \sum_{h=0}^{H-1}D^2_H(P^{c_i}_\star(\cdot|s^i_{h},a^i_{h}),P^{c_i}(\cdot|s^i_{h},a^i_{h}))
        \\
        \leq & 
        \sum_{i=1}^{t-1} \sum_{h=0}^{H-1} \log\left( \frac{1}{P^{c_i}(s^i_{h+1}|s^i_h,a^i_h)}\right) 
        -
        \sum_{i=1}^{t-1} \sum_{h=0}^{H-1} \log\left( \frac{1}{P_\star^{c_i}(s^i_{h+1}|s^i_h,a^i_h)}\right) 
        + 2H \log(TH |\Fp|/\delta).
    \end{align*}
    simultaneously, for all $t \geq 1$ and $P \in \Fp$. 
\end{lemma}

\begin{proof}
    Fix some $t \ge 1$ and $P \in \Fp$. We have with probability at least $1-\frac{\delta}{T|\Fp|}$ that
    \begingroup
    \allowdisplaybreaks
    \begin{align*}
        &
        \sum_{i=1}^{t-1}
        \sum_{h=0}^{H-1}D^2_H(P^{c_i}_\star(\cdot|s^{i}_{h},a^{i}_{h}),P^{c_i}(\cdot|s^{i}_{h},a^{i}_{h}))
        \\
        &
        =
        \sum_{h=0}^{H-1}
        \sum_{i=1}^{t-1}
        D^2_H(P^{c_i}_\star(\cdot|s^{i}_{h},a^{i}_{h}),P^{c_i}(\cdot|s^{i}_{h},a^{i}_{h}))
        \\
        \tag{\citealp[Lemma A.14]{foster2021statistical}}
        & \leq
        \sum_{h=0}^{H-1}\left(\sum_{i=1}^{t-1}\log\left(\frac{P^{c_i}_\star(s^i_{h+1}|s^i_h,a^i_h)}{{P}^{c_i}(s^i_{h+1}|s^i_h,a^i_h)}\right) + 2 \log(HT |\Fp|/\delta)\right)
        \\
        & =
        \sum_{i=1}^{t-1} \sum_{h=0}^{H-1} \log\left(\frac{1}{{P}^{c_i}(s^i_{h+1}|s^i_h,a^i_h)} \right) -  \sum_{i=1}^{t-1} \sum_{h=0}^{H-1}\log\left(\frac{1}{P^{c_i}_\star(s^i_{h+1}|s^i_h,a^i_h)}\right) + 2 H\log(HT |\Fp|/\delta)
        .
    \end{align*}
    \endgroup
    The filtration used in $(i)$ is over the history up to time $t$, $\Hist_{t-1} = (\sigma^1, \ldots,\sigma^{t-1})$.
    Now, by taking a union bound over every $t = 1.\ldots, T$ and $P \in \Fp$, we obtain the lemma.
\end{proof}

\begin{corollary}[restatement of \cref{corl:hellinger-distance-bound-w.h.p-main}]
    Let $\widehat{P}_t \in \Fp$ be the log loss minimizer in \cref{alg:UCCRL-main}.
    For any $\delta \in (0,1)$ it holds that with probability at least $1-\delta$ we have
    \begin{align*}
        &{\E}_c \brk[s]*{
        \sum_{i=1}^{t-1}
        \mathop{\E}_{\pi_i(c;\cdot), P^{c}_\star} \Bigg[
        \sum_{h=0}^{H-1}D^2_H(P^{c}_\star(\cdot|s_{h},a_{h}),\widehat{P}_t^{c}(\cdot|s_{h},a_{h}))
        \Bigg| s_0\Bigg]}
        \leq
        2H \log(TH  |\Fp|/\delta)
        ,
        &&
        \forall 1 \le t \le T
        .
    \end{align*}
\end{corollary}

\begin{proof}
    By our assumption that $P_\star \in \Fp$, and $\widehat{P}_t$ is the log loss minimizer at time $t$, it holds that
    \[
        \sum_{i=1}^{t-1} \sum_{h=0}^{H-1} \log\left(\frac{1}{\widehat{P}^{c_i}_t(s^i_{h+1}|s^i_h,a^i_h)} \right) -  \sum_{i=1}^{t-1} \sum_{h=0}^{H-1}\log\left(\frac{1}{P^{c_i}_\star(s^i_{h+1}|s^i_h,a^i_h)}\right)
        \leq 0.
    \]
    Thus, the corollary immediately follows by
    \cref{lemma:transition-to-oracle-rgret-dynamics}. 
\end{proof}

\begin{lemma}[Stability error of log-loss oracle, restatement of \cref{lemma:log-loss-regret}]
    Let $\widehat{P}_i \in \Fp$ denote the log-loss minimizer at round $i \in [T]$.
    For any $\delta \in (0,1)$ it holds with probability at least $1-\delta$ that
    \begin{align*}
        &\E_c \brk[s]*{ \sum_{i=1}^{t-1}
        \E_{\pi_i(c;\cdot),P^c_\star} \brk[s]*{
        \sum_{h=0}^{H-1} D^2_H (P^c_\star(\cdot|s_h,a_h), \widehat{P}^c_i(s_h,a_h))
        \Bigg| s_0}}
        \leq
        112 H \dEP
        \log^2(2TH  |\Fp|/\delta)
    \end{align*}
    simultaneously, for all $t \geq 1$,
    where $\dEP \ge \dHE(\Fp, D_H, T^{-1/2})$, the Eluder dimension of $\Fp$ at scale $T^{-1/2}$.
\end{lemma}

\begin{proof}
    Let $\beta = 2H \log(2TH  |\Fp|/\delta)$ and define
    \begin{align*}
        \Fp_t
        =
        \brk[c]*{
        P\in \Fp
        \;:\;
        \sum_{i=1}^{t-1} \sum_{h=0}^{H-1} D_H^2(P^{c_i}(\cdot | s_h^i,a_h^i), \widehat{P}^{c_i}_t(\cdot | s_h^i, a_h^i)
        )
        \le
        \beta
        }
        .
\end{align*}
    Now, suppose that \cref{lemma:realized-log-loss-regret} holds with $\delta / 2$. Then, since $\widehat{P}_t$ is the log-loss minimizer, we have that $P_\star \in \Fp_t$ for all $1 \le t \le T$.
    Next, recalling that
    \begin{align*}
        w_{\Fp_t}(s,a,c)
        =
        \sup_{P,P' \in \Fp_t} D_H({P^c}(\cdot|s,a), {P'^c}(\cdot|s,a))
        ,
    \end{align*}
    thus, we have that
    \begin{align*}
        \sum_{i=1}^{t-1} \sum_{h=0}^{H-1} D_H^2(P_\star^{c_i}(\cdot | s_h^i,a_h^i), \widehat{P}^{c_i}_i(\cdot | s_h^i, a_h^i)
        \le
        \sum_{i=1}^{t-1} \sum_{h=0}^{H-1}
        \brk{w_{\Fp_i}(s_h^i, a_h^i, c_i)}^2
        .
    \end{align*}
    Applying \cref{lemma:russo-lemma2} we get that
    \begin{align*}
        \sum_{i=1}^{t-1} \sum_{h=0}^{H-1} D_H^2(P_\star^{c_i}(\cdot | s_h^i,a_h^i), \widehat{P}^{c_i}_i(\cdot | s_h^i, a_h^i)
        &
        \le
        24 \dEP H \beta \log T
        \\
        &
        \le
        48 H \dEP
        \log^2(2TH  |\Fp|/\delta)
        .
    \end{align*}
    Finally, we apply \cref{lemma:multiplicative-concentration} to get that with probability at least $1-\delta/2$
    \begin{align*}
        \E_c &\brk[s]*{ \sum_{i=1}^{t-1}
        \E_{\pi_i(c;\cdot),P^c_\star} \brk[s]*{
        \sum_{h=0}^{H-1} D^2_H (P^c_\star(\cdot|s_h,a_h), \widehat{P}^c_i(s_h,a_h))
        \Big| s_0}}
        \\
        &
        =
        \sum_{i=1}^{t-1} \sum_{h=0}^{H-1} \E\brk[s]*{D_H^2(P_\star^{c_i}(\cdot | s_h^i,a_h^i), \widehat{P}^{c_i}_i(\cdot | s_h^i, a_h^i) \Big| 
        \pi_i, s_0}
        \\
        &
        \le
        2\sum_{i=1}^{t-1} \sum_{h=0}^{H-1} D_H^2(P_\star^{c_i}(\cdot | s_h^i,a_h^i), \widehat{P}^{c_i}_i(\cdot | s_h^i, a_h^i)
        +
        16 H \log \frac{2 T H}{\delta}
        .
    \end{align*}
    Taking a union bound and combining the last two inequalities concludes the proof.
\end{proof}

\subsection{Confidence Bounds (Step 2)}
\label{sec:proofs-confidence-bounds}
In the following analysis, we use an occupancy measures-based representation of the value function.
Recall the definition of the 
\emph{occupancy measures}~\citep{zimin2013online}. 
For any non-contextual policy $\pi$ and dynamics $P$,
let $q_h(s,a | \pi, P)$ denote the probability of reaching state $s\in S$ and performing action $a\in A$ at time $h \in [H]$ of an episode generated using  policy $\pi$ and dynamics $P$. 

Using this notation, the value function of any policy $\pi$ with respect to the MDP $(S,A,P,r,s_0\linebreak[1],H)$
can be represented as follows.
\begin{align}
\label{eq:value-occupancy-representation}
    V^{\pi}_{M} (s_0)
    = 
    \sum_{h=0}^{H-1} 
    \sum_{s \in S_h}
    \sum_{a \in A}
    q_h(s,a | \pi, P)\cdot r(s, a)
    .
\end{align}
Thus, the following is an immediate corollary of~\cref{lemma:value-change-of-measure-main}.
\begin{corollary}
\label{corl:occupancy-change-of-measure}
    For any (non-contextual) policy $\pi$, two dynamics $P$ and $\widehat{P}$, and rewards function $r$ that is bounded in $[0,1]$ it holds that
    \begin{align*}
        \sum_{h=0}^{H-1}
        \sum_{s \in S_h}
        \sum_{a \in A}
        q_h(s,a| \pi,\widehat{P})\cdot
        r(s,a)
        &
        \leq 
        3\sum_{h=0}^{H-1}
        \sum_{s \in S_h}
        \sum_{a \in A}
        q_h(s,a| \pi,P)\cdot
        r(s,a)
        \\
        & +
        9H^2
        \sum_{h=0}^{H-1}
        \sum_{s \in S_h}
        \sum_{a \in A}
        q_h(s,a| \pi,P)\cdot
         D^2_H(P(\cdot|s,a), \widehat{P}(\cdot|s,a)) 
        .
    \end{align*}
\end{corollary}


We are now ready to prove the confidence bounds. Recall the reward bonuses $b_{t,h}^R, b_{t,h}^P$ defined in \cref{eq:reward-bonuses}.
\begin{lemma}[restatement of \cref{lemma:CB-rewards-approx-p-main}]
    Let $P_\star$ and $f_\star$ be the true context dependent dynamics and rewards. Let $\widehat{P}_t$ and $\widehat{f}_t$ be the approximated context-dependent dynamics and  rewards at round $t$.
    Then, 
    for any $t\geq 1$, and context-dependent policy $\pi \in \Pi_\C$ the following holds.
    \begingroup \allowdisplaybreaks
    \begin{align*}
        &
        \abs*{
        \E_c\left[ V^{\pi(c;\cdot)}_{\M^{(f_\star,\widehat{P}_t)}(c)}(s_0) 
        \right]
        -
         \E_c\left[ V^{\pi(c;\cdot)}_{\M^{(\widehat{f}_t, \widehat{P}_t)}(c)}(s_0) \right]
        }
        \\
        &
        \hspace{10em}
        \leq 
        \E_c\left[
        \sum_{h=0}^{H-1}
        \sum_{s_h \in S^c_h} \sum_{a_h \in A}
        q_h(s_h,a_h|\pi(c;\cdot),\widehat{P}^c_t)
        b_{t,h}^R(c,s_h,a_h)
        \right]
        \\
        &\hspace{10em} +
        \frac{3}{2\betaR}
        \E_c \left[ \sum_{i=1}^{t-1} 
        \mathop{\E}_{\pi_i(c;\cdot), P^{c}_\star} \Bigg[\sum_{h=0}^{H-1} 
        \left( f_\star(c,s_h,a_h) - \widehat{f}_t(c,s_h,a_h)\right)^2 \Bigg|
        s_0
        \right]\Bigg]
        \\
        & \hspace{10em} +
        \frac{9 H^2}{2\betaR} \E_c \left[ \sum_{i=1}^{t-1} 
        \mathop{\E}_{\pi_i(c;\cdot), P^{c}_\star} \Bigg[\sum_{h=0}^{H-1}D^2_H(P^{c}_\star(\cdot|s_{h},a_{h}), \widehat{P}^{c}_i(\cdot|s_{h},a_{h}))\Bigg| s_0\Bigg]\right]
        \\
        & \hspace{10em}+
        \frac{H}{2 \betaR}
        .
    \end{align*}
    \endgroup
\end{lemma}
\begin{proof}
We have that
\begingroup\allowdisplaybreaks
\begin{align*}
    &\left| 
        \E_c\left[ V^{\pi(c;\cdot)}_{\M^{(f_\star,\widehat{P}_t)}(c)}(s_0) \right]
        -
        \E_c\left[ V^{\pi(c;\cdot)}_{\M^{(\widehat{f}_t, \widehat{P}_t)}(c)}(s_0) \right]
    \right|
    \\
    \tag{By linearity of expectation}
    &
    =
    \left| 
        \E_c\left[ V^{\pi(c;\cdot)}_{\M^{(f_\star,\widehat{P}_t)}(c)}(s_0)
         -
        V^{\pi(c;\cdot)}_{\M^{(\widehat{f}_t, \widehat{P}_t)}(c)}(s_0) \right]
    \right|
    \\
    \tag{\cref{eq:value-occupancy-representation}}
    &
    = 
    \left| 
        \E_c\left[
        \sum_{h=0}^{H-1} \sum_{s_h \in S^c_h} \sum_{a_h \in A} q_h(s_h,a_h| \pi(c;\cdot), \widehat{P}^c_t)\cdot \left( \widehat{f}_t(c,s_h,a_h) - f_\star(c,s_h,a_h) \right)
        \right]
    \right|
    \\
    \tag{Triangle ineq.}
    &
    \leq
    \E_c\left[
    \sum_{h=0}^{H-1} \sum_{s_h \in S^c_h} \sum_{a_h \in A} q_h(s_h,a_h| \pi(c;\cdot), \widehat{P}^c_t)\cdot \left| \widehat{f}_t(c,s_h,a_h) - f_\star(c,s_h,a_h) \right|
    \right]
    \\
    &
    =
    \E_c\left[
    \sum_{h=0}^{H-1}
    \min\left\{
    1
    ,
    \sum_{s_h \in S^c_h} \sum_{a_h \in A} q_h(s_h,a_h| \pi(c;\cdot), \widehat{P}^c_t)\cdot \left| \widehat{f}_t(c,s_h,a_h) - f_\star(c,s_h,a_h) \right|
    \right\}
    \right]
    \\
    \tag{AM-GM}
    &
    \leq
    \E_c\Bigg[
    \sum_{h=0}^{H-1}
    \min\Bigg\{
    1
    ,
    \sum_{s_h \in S^c_h} \sum_{a_h \in A}
    \frac{\betaR}{2}
    \frac{q_h(s_h,a_h|\pi(c;\cdot),\widehat{P}^c_t)}{1 + \sum_{i=1}^{t-1}q_h(s_h,a_h| \pi_i(c;\cdot), \widehat{P}^c_i)}
    \\
    &
    +
    \frac{1}{2\betaR}
    q_h(s_h,a_h|\pi(c;\cdot),\widehat{P}^c_t)\left(1 + \sum_{i=1}^{t-1}q_h(s_h,a_h| \pi_i(c;\cdot),\widehat{P}^c_i)\right) \left( f_\star(c,s_h,a_h) - \widehat{f}_t(c,s_h,a_h)\right)^2 
    \Bigg\}
    \Bigg]
    \\
    &
    \leq 
    \E_c\left[
    \sum_{h=0}^{H-1}
    \min\left\{
    1
    ,
    \sum_{s_h \in S^c_h} \sum_{a_h \in A}
    \frac{\betaR}{2}
    \frac{q_h(s_h,a_h|\pi(c;\cdot),\widehat{P}^c_t)}{1 + \sum_{i=1}^{t-1}q_h(s_h,a_h| \pi_i(c;\cdot), \widehat{P}^c_i)}
    \right\}
    \right]
    \\
    &
    +
    \frac{H}{2 \betaR}
    +
    \frac{1}{2\betaR}
    \E_c\left[\sum_{i=1}^{t-1}\sum_{h=0}^{H-1} \sum_{s_h \in S^c_h}\sum_{a_h \in A}
    q_h(s_h,a_h| \pi_i(c;\cdot),\widehat{P}^c_i) \left( f_\star(c,s_h,a_h) - \widehat{f}_t(c,s_h,a_h)\right)^2 
    \right]
    \\
    &
    \le
    \E_c\left[
    \sum_{h=0}^{H-1}
    \sum_{s_h \in S^c_h} \sum_{a_h \in A}
    q_h(s_h,a_h|\pi(c;\cdot),\widehat{P}^c_i)
    \min\left\{
    1
    ,\frac{\betaR/2}{1 + \sum_{i=1}^{t-1}q_h(s_h,a_h| \pi_i(c;\cdot), \widehat{P}^c_i)}
    \right\}
    \right]
    \\
    \tag{\cref{corl:occupancy-change-of-measure}}
    &
    +
    \frac{3}{2\betaR}
     \E_c \left[ \sum_{i=1}^{t-1} 
    \mathop{\E}_{\pi_i(c;\cdot), P^{c}_\star} \Bigg[
    \sum_{h=0}^{H-1} 
    \left( f_\star(c,s_h,a_h) - \widehat{f}_t(c,s_h,a_h)\right)^2 
    \Big| s_0
    \right]\Bigg]
    \\
    & 
    +
    \frac{9 H^2}{2\betaR} \E_c \left[ \sum_{i=1}^{t-1} 
    \mathop{\E}_{\pi_i(c;\cdot), P^{c}_\star} \Bigg[\sum_{h=0}^{H-1}D^2_H(P^{c}_\star(\cdot|s_{h},a_{h}), \widehat{P}^{c}_i(\cdot|s_{h},a_{h}))\Bigg| s_0\Bigg]\right]
    +
    \frac{H}{2 \betaR}
    ,
\end{align*}
\endgroup
and the lemma follows by $b^R_{t,h}$ definition.
\end{proof}

\begin{lemma}[restatement of \cref{lemma:CB-dynamics-true-r-main}]
    Let $P_\star$ and $f_\star$ be the true context dependent dynamics and rewards. Let $\widehat{P}_t$ be the approximated context-dependent dynamics at round $t$.
    Then, 
    for any $t\geq 1$, and context-dependent policy $\pi \in \Pi_\C$ we have
    \begingroup
    \allowdisplaybreaks
    \begin{align*}  
        &
        \abs*{
        \E_c\brk[s]*{V^{\pi(c;\cdot)}_{\M^{({f}_\star,P_\star)}(c)}(s_0) 
        }
        -
        \E_c\brk[s]*{
        V^{\pi(c;\cdot)}_{\M^{({f}_\star, \widehat{P}_t)}(c)}(s_0)}
        }
        \leq 
        \E_c\brk[s]*{
        \sum_{h=0}^{H-1}
        \sum_{s \in S_h^c}
        \sum_{a \in A}
        q_h(s,a| \pi(c;\cdot), \widehat{P}^c_t )
        \cdot b^P_{t,h}(c,s,a)
        }
        \\
        & \hspace{10em}+
        \frac{6 H}{ \betaP}
        \E_c \left[\sum_{i=1}^{t-1} 
        \mathop{\E}_{\pi_i(c;\cdot), P^{c}_\star} \Bigg[
        \sum_{h=0}^{H-1}
         D^2_H(P^{c}_\star(\cdot|s_{h},a_{h}), \widehat{P}^{c}_t(\cdot|s_{h},a_{h}))
        \Bigg| s_0\Bigg]\right]
        \\
        & \hspace{10em}+
        \frac{ 18 H^3}{\betaP}
        \E_c \left[\sum_{i=1}^{t-1} 
        \mathop{\E}_{\pi_i(c;\cdot), P^{c}_\star} \Bigg[
        \sum_{h=0}^{H-1}
         D^2_H(P^{c}_\star(\cdot|s_{h},a_{h}), \widehat{P}^{c}_i(\cdot|s_{h},a_{h}))
        \Bigg| s_0\Bigg]\right]
        \\
        & \hspace{10em}+
        \frac{H^2}{2 \betaP}
        .
    \end{align*}
    \endgroup
\end{lemma}
\begin{proof}
    The following holds 
    for any $t \geq 1$ and a context-dependent policy $\pi \in \Pi_\C$.
    \begingroup
    \allowdisplaybreaks
    \begin{align*}          
       &\left|\E_c[V^{\pi(c;\cdot)}_{\M^{({f}_\star,P_\star)}(c)}(s_0)] - \E_c[V^{\pi(c;\cdot)}_{\M^{({f}_\star, \widehat{P}_t)}(c)}(s_0)]\right|
        \\
         \tag{By linearity of expectation}
        &
        = 
        \left|\E_c\left[V^{\pi(c;\cdot)}_{\M^{(f_\star,P_\star)}(c)}(s_0) - V^{\pi(c;\cdot)}_{\M^{(f_\star, \widehat{P}_t)}(c)}(s_0)\right]\right|
        \\
        \tag{\cref{lemma:val-diff}}
        &
        =
        \left|
        \E_c \left[ 
        \E_{\pi(c;\cdot),\widehat{P}^c_t}
        \left[
        \sum_{h=0}^{H-1} 
        \sum_{s' \in S}
        (P^c_\star(s'|s_h,a_h) - \widehat{P}^c_t(s'|s_h,a_h))V^{\pi(c;\cdot)}_{\M^{(f_\star, {P}_\star)},h+1}(s') \right] \Bigg| s_0 \right]
        \right|
        \\
        &
        =
        \tag{\cref{eq:value-occupancy-representation}}
        \left|
        \E_c 
        \left[
        \sum_{h=0}^{H-1}
        \sum_{s_h \in S_h^c}
        \sum_{a_h \in A}
        q_h(s_h,a_h| \pi(c;\cdot), \widehat{P}^c_t ) 
        \sum_{s' \in S}
        (P^c_\star(s'|s_h,a_h) - \widehat{P}^c_t(s'|s_h,a_h))V^{\pi(c;\cdot)}_{\M^{(f_\star, {P}_\star)},h+1}(s') \right] \right|
        \\
        &
        \le
        \tag{Triangle inequality}
        \E_c 
        \left[
        \sum_{h=0}^{H-1}
        \sum_{s_h \in S_h^c}
        \sum_{a_h \in A}
        q_h(s_h,a_h| \pi(c;\cdot), \widehat{P}^c_t ) 
        \left|
        \sum_{s' \in S}
        (P^c_\star(s'|s_h,a_h) - \widehat{P}^c_t(s'|s_h,a_h))V^{\pi(c;\cdot)}_{\M^{({f}_\star, {P}_\star)},h+1}(s') 
        \right|
        \right] 
        \\
        \tag{$f_\star \in [0,1]$, $V^{\pi(c;\cdot)}_{\M^{({f}_\star, {P}_\star)},h}(s') \in [0,H]$ for all $h \in [H]$ and $s' \in S$}
        &
        \le
        H
        \E_c 
        \left[
        \sum_{h=0}^{H-1}
        \min\brk[c]*{1,
        \sum_{s_h \in S_h^c}
        \sum_{a_h \in A}
        q_h(s_h,a_h| \pi(c;\cdot), \widehat{P}^c_t ) 
        \sum_{s' \in S}
        \left|
        P^c_\star(s'|s_h,a_h) - \widehat{P}^c_t(s'|s_h,a_h)  
        \right|
        }
        \right]
        \\
        \tag{AM-GM}
        &
        \le
        H
        \E_c 
        \Bigg[
        \sum_{h=0}^{H-1}
        \min\Bigg\{1,
        \sum_{s_h \in S_h^c}
        \sum_{a_h \in A}
        \Bigg(
        \frac{\betaP}{2}
        \frac{q_h(s_h,a_h| \pi(c;\cdot), \widehat{P}^c_t )}
        {1 +\sum_{i=1}^{t-1} q_h(s_h,a_h | \pi_i(c;\cdot),\widehat{P}^c_i)}
        \\
        &
        +
        \frac{q_h(s_h,a_h| \pi(c;\cdot), \widehat{P}^c_t )}{2 \betaP}
        \brk{1 +\sum_{i=1}^{t-1} q_h(s_h,a_h | \pi_i(c;\cdot),\widehat{P}^c_i)}
        \brk*{
        \sum_{s' \in S}
        \left|
        P^c_\star(s'|s_h,a_h) - \widehat{P}^c_t(s'|s_h,a_h)
        \right|
        }^2
        \Bigg)
        \Bigg\}
        \Bigg]
        \\
        &
        \le
        \E_c 
        \left[
        \sum_{h=0}^{H-1}
        \sum_{s_h \in S_h^c}
        \sum_{a_h \in A}
        q_h(s_h,a_h| \pi(c;\cdot), \widehat{P}^c_t )
        H\min\brk[c]*{
         1
        ,
        \frac{\betaP  / 2}
        {1 +\sum_{i=1}^{t-1} q_h(s_h,a_h | \pi_i(c;\cdot),\widehat{P}^c_i)}
        }
        \right]
        +
        \frac{H^2}{2 \betaP}
        \\
        & 
        +
        \frac{H}{2 \betaP}
        \E_c 
        \Bigg[
        \sum_{i=1}^{t-1} 
        \sum_{h=0}^{H-1}
        \sum_{s_h \in S_h^c}
        \sum_{a_h \in A}
        q_h(s_h,a_h | \pi_i(c;\cdot),\widehat{P}^c_i)
        \brk*{
        \sum_{s' \in S}
        \left|
        P^c_\star(s'|s_h,a_h) - \widehat{P}^c_t(s'|s_h,a_h)
        \right|
        }^2
        \Bigg]
        \\
        &
        \le
        \E_c 
        \left[
        \sum_{h=0}^{H-1}
        \sum_{s_h \in S_h^c}
        \sum_{a_h \in A}
        q_h(s_h,a_h| \pi(c;\cdot), \widehat{P}^c_t )
        H\min\brk[c]*{
         1
        ,
        \frac{\betaP  / 2}
        {1 +\sum_{i=1}^{t-1} q_h(s_h,a_h | \pi_i(c;\cdot),\widehat{P}^c_i)}
        }
        \right]
        +
        \frac{H^2}{2 \betaP}
        \\
        &
        +
        \frac{3 H}{2 \betaP}
        \E_c 
        \Bigg[
        \sum_{i=1}^{t-1} 
        \sum_{h=0}^{H-1}
        \sum_{s_h \in S_h^c}
        \sum_{a_h \in A}
        q_h(s_h,a_h | \pi_i(c;\cdot),{P}_\star^c)
        \brk*{
        \sum_{s' \in S}
        \left|
        P^c_\star(s'|s_h,a_h) - \widehat{P}^c_t(s'|s_h,a_h)
        \right|
        }^2
        \Bigg]
        \\
        \tag{\cref{corl:occupancy-change-of-measure}}
        &
        +
        \frac{ 18 H^3}{\betaP}
        \E_c \left[\sum_{i=1}^{t-1} 
        \mathop{\E}_{\pi_i(c;\cdot), P^{c}_\star} \Bigg[
        \sum_{h=0}^{H-1}
         D^2_H(P^{c}_\star(\cdot|s_{h},a_{h}), \widehat{P}^{c}_i(\cdot|s_{h},a_{h}))
        \Bigg| s_0\Bigg]\right]
        \\
        &
        \le
        \E_c 
        \left[
        \sum_{h=0}^{H-1}
        \sum_{s_h \in S_h^c}
        \sum_{a_h \in A}
        q_h(s_h,a_h| \pi(c;\cdot), \widehat{P}^c_t )
        H\min\brk[c]*{
         1
        ,
        \frac{\betaP  / 2}
        {1 +\sum_{i=1}^{t-1} q_h(s_h,a_h | \pi_i(c;\cdot),\widehat{P}^c_i)}
        }
        \right]
        +
        \frac{H^2}{2 \betaP}
        \\
        \tag{$\mathrm{TV}^2 \le  4 D_H^2$}
        & 
        +
        \frac{6 H}{ \betaP}
        \E_c 
        \Bigg[
        \sum_{i=1}^{t-1} 
        \sum_{h=0}^{H-1}
        \sum_{s_h \in S_h^c}
        \sum_{a_h \in A}
        q_h(s_h,a_h | \pi_i(c;\cdot),{P}_\star^c)
        D^2_H(P^c_\star(\cdot|s_h,a_h) , \widehat{P}^c_t(\cdot|s_h,a_h))
        \Bigg]
        \\
        & 
        +
        \frac{ 18 H^3}{\betaP}
        \E_c \left[\sum_{i=1}^{t-1} 
        \mathop{\E}_{\pi_i(c;\cdot), P^{c}_\star} \Bigg[
        \sum_{h=0}^{H-1}
         D^2_H(P^{c}_\star(\cdot|s_{h},a_{h}), \widehat{P}^{c}_i(\cdot|s_{h},a_{h}))
        \Bigg| s_0\Bigg]\right]
        \\
        &
        =
        \E_c 
        \left[
        \sum_{h=0}^{H-1}
        \sum_{s_h \in S_h^c}
        \sum_{a_h \in A}
        q_h(s_h,a_h| \pi(c;\cdot), \widehat{P}^c_t )
        b^P_{t,h}(c,s_h,a_h)
        \right]
        +
        \frac{H^2}{2 \betaP}
        \\
        &
        +
        \frac{6 H}{ \betaP}
        \E_c \left[\sum_{i=1}^{t-1} 
        \mathop{\E}_{\pi_i(c;\cdot), P^{c}_\star} \Bigg[
        \sum_{h=0}^{H-1}
         D^2_H(P^{c}_\star(\cdot|s_{h},a_{h}), \widehat{P}^{c}_t(\cdot|s_{h},a_{h}))
        \Bigg| s_0\Bigg]\right]
        \\
        &
        +
        \frac{ 18 H^3}{\betaP}
        \E_c \left[\sum_{i=1}^{t-1} 
        \mathop{\E}_{\pi_i(c;\cdot), P^{c}_\star} \Bigg[
        \sum_{h=0}^{H-1}
         D^2_H(P^{c}_\star(\cdot|s_{h},a_{h}), \widehat{P}^{c}_i(\cdot|s_{h},a_{h}))
        \Bigg| s_0\Bigg]\right]
        .
    \end{align*}
    \endgroup
\end{proof}

\begin{corollary}[restatement of \cref{corl:UCB-value-main}]
    Under the terms of \cref{lemma:CB-rewards-approx-p-main,lemma:CB-dynamics-true-r-main}, the following holds with probability at least $1-3\delta/4$ simultaneously for all $t \geq 1$ and $\pi \in \Pi_\C$:
    \begingroup \allowdisplaybreaks
    \begin{align*}
        &
        \abs*{
        \E_c\left[ V^{\pi(c;\cdot)}_{\M^{(f_\star,{P}_\star)}(c)}(s_0) 
        \right]
        -
         \E_c\left[ V^{\pi(c;\cdot)}_{\M^{(\widehat{f}_t, \widehat{P}_t)}(c)}(s_0) \right]
        }
        \\
        &
        \le
        \E_c\left[
        \sum_{h=0}^{H-1}
        \sum_{s_h \in S^c_h} \sum_{a_h \in A}
        q_h(s_h,a_h|\pi(c;\cdot),\widehat{P}^c_t)
        \brk*{
        b_{t,h}^R(c,s_h,a_h)
        +
        b_{t,h}^P(c,s_h,a_h)
        }
        \right]
        \\
        &
        +
        \frac{2029 H^4 \dEP}{\betaP}
        \log^2(8TH  |\Fp|/\delta)
        +
        \frac{504 H^3 \dEP }{\betaR} 
        \log^2(64T^4 H  |\F| |\Fp|/\delta^2)
        .
    \end{align*}
    \endgroup
\end{corollary}


\begin{proof}
    We begin by taking a union bound on the events of \cref{corl:reward-distance-bound-w.h.p-main,corl:hellinger-distance-bound-w.h.p-main,lemma:log-loss-regret} to get that with probability at least $1-3\delta/4$, simultaneously for all $t \ge 1$
    \begin{align}
        \nonumber
        &{\E}_c \brk[s]*{
        \sum_{i=1}^{t-1}
        \mathop{\E}_{\pi_i(c;\cdot), P^{c}_\star} \Bigg[
        \sum_{h=0}^{H-1}
        \brk*{
        \widehat{f}_t(c, s_h, a_h) - f_\star(c, s_h, a_h)
        }^2
        \Bigg| s_0\Bigg]}
        \leq
        68H \log(8 T^3 |\F|/\delta)
        \\
        \label{eq:oracle-bounds}
        &{\E}_c \brk[s]*{
        \sum_{i=1}^{t-1}
        \mathop{\E}_{\pi_i(c;\cdot), P^{c}_\star} \Bigg[
        \sum_{h=0}^{H-1}D^2_H(P^{c}_\star(\cdot|s_{h},a_{h}),\widehat{P}_t^{c}(\cdot|s_{h},a_{h}))
        \Bigg| s_0\Bigg]}
        \leq
        2H \log(4TH  |\Fp|/\delta)
        \\
        \nonumber
        &{\E}_c \brk[s]*{
        \sum_{i=1}^{t-1}
        \mathop{\E}_{\pi_i(c;\cdot), P^{c}_\star} \Bigg[
        \sum_{h=0}^{H-1}D^2_H(P^{c}_\star(\cdot|s_{h},a_{h}),\widehat{P}_i^{c}(\cdot|s_{h},a_{h}))
        \Bigg| s_0\Bigg]}
        \leq
        112 H \dEP
        \log^2(8TH  |\Fp|/\delta)
        .
    \end{align}
    Assuming this event holds, we get that for all $t \geq 1$ and context-dependent policy $\pi \in \Pi_\C$.
     \begingroup\allowdisplaybreaks
     \begin{align*}
        & \abs*{
        \E_c\left[ V^{\pi(c;\cdot)}_{\M^{(f_\star,{P}_\star)}(c)}(s_0) \right]
        -
        \E_c\left[ V^{\pi(c;\cdot)}_{\M^{(\widehat{f}_t, \widehat{P}_t)}(c)}(s_0) \right]
        }
        \\
        \tag{By triangle inequality}
        & 
        \le
        \abs*{
        \E_c\left[ V^{\pi(c;\cdot)}_{\M^{(f_\star,{P}_\star)}(c)}(s_0) \right]
        -
        \E_c\left[ V^{\pi(c;\cdot)}_{\M^{({f}_\star, \widehat{P}_t)}(c)}(s_0) \right]
        }
        \\
        &
        \hspace{1em}
        +
        \abs*{
        \E_c\left[ V^{\pi(c;\cdot)}_{\M^{(f_\star,\widehat{P}_t)}(c)}(s_0) \right]
        -
        \E_c\left[ V^{\pi(c;\cdot)}_{\M^{(\widehat{f}_t, \widehat{P}_t)}(c)}(s_0) \right]
        }
        \\
        \tag{\cref{lemma:CB-rewards-approx-p-main,lemma:CB-dynamics-true-r-main}}
        &
        \le
        \E_c\left[
        \sum_{h=0}^{H-1}
        \sum_{s_h \in S^c_h} \sum_{a_h \in A}
        q_h(s_h,a_h|\pi(c;\cdot),\widehat{P}^c_t)
        \brk*{
        b_{t,h}^R(c,s_h,a_h)
        +
        b_{t,h}^P(c,s_h,a_h)
        }
        \right]
        \\
        &
        \hspace{1em}
        +
        \frac{3}{2\betaR}
        \E_c \left[ \sum_{i=1}^{t-1} 
        \mathop{\E}_{\pi_i(c;\cdot), P^{c}_\star} \Bigg[\sum_{h=0}^{H-1} 
        \left( f_\star(c,s_h,a_h) - \widehat{f}_t(c,s_h,a_h)\right)^2 \Bigg|
        s_0
        \right]\Bigg]
        \\
        & 
        \hspace{1em}
        +
        \frac{6H}{\betaP}
        \E_c \left[ \sum_{i=1}^{t-1} 
        \mathop{\E}_{\pi_i(c;\cdot), P^{c}_\star} \Bigg[\sum_{h=0}^{H-1}D^2_H(P^{c}_\star(\cdot|s_{h},a_{h}), \widehat{P}^{c}_t(\cdot|s_{h},a_{h}))\Bigg| s_0\Bigg]\right]
        \\
        & 
        \hspace{1em}
        +
        \brk*{
        \frac{9 H^2}{2\betaR} 
        +
        \frac{18 H^3}{\betaP}
        }
        \E_c \left[ \sum_{i=1}^{t-1} 
        \mathop{\E}_{\pi_i(c;\cdot), P^{c}_\star} \Bigg[\sum_{h=0}^{H-1}D^2_H(P^{c}_\star(\cdot|s_{h},a_{h}), \widehat{P}^{c}_i(\cdot|s_{h},a_{h}))\Bigg| s_0\Bigg]\right]
        \\
        &
        \hspace{1em}
        +
        \frac{H}{2 \betaR}
        +
        \frac{H^2}{2 \betaP}
        \\
        \tag{\cref{eq:oracle-bounds}}
        &
        \leq
        \E_c\left[
        \sum_{h=0}^{H-1}
        \sum_{s_h \in S^c_h} \sum_{a_h \in A}
        q_h(s_h,a_h|\pi(c;\cdot),\widehat{P}^c_t)
        \brk*{
        b_{t,h}^R(c,s_h,a_h)
        +
        b_{t,h}^P(c,s_h,a_h)
        }
        \right]
        \\
        &
        \hspace{1em}
        +
        \frac{3}{2\betaR}
        68H \log(8 T^3 |\F|/\delta)
        +
        \frac{6 H}{\betaP}
        2H \log(4TH  |\Fp|/\delta)
        \\
        &
        \hspace{1em}
        +
        \brk*{
        \frac{9 H^2}{2\betaR} 
        +
        \frac{18 H^3}{\betaP}
        }
         112 H \dEP
        \log^2(8TH  |\Fp|/\delta)
        \\
        &
        \hspace{1em}
        +
        \frac{H}{2 \betaR}
        +
        \frac{H^2}{2 \betaP}
        \\
        &
        \leq
        \E_c\left[
        \sum_{h=0}^{H-1}
        \sum_{s_h \in S^c_h} \sum_{a_h \in A}
        q_h(s_h,a_h|\pi(c;\cdot),\widehat{P}^c_t)
        \brk*{
        b_{t,h}^R(c,s_h,a_h)
        +
        b_{t,h}^P(c,s_h,a_h)
        }
        \right]
        \\
        &
        \hspace{1em}
        +
        \frac{2029 H^4 \dEP}{\betaP}
        \log^2(8TH  |\Fp|/\delta)
        +
        \frac{504 H^3 \dEP }{\betaR} 
        \log^2(64T^4 H  |\F| |\Fp|/\delta^2)
        ,
     \end{align*}
     \endgroup
     as stated.
\end{proof}

\subsection{Establishing Optimism Lemmas (Step 3)}
\label{sec:proofs-optimism}

\begin{lemma}[restatement of \cref{lemma:opt-in-expectation-main}]
    Let $\pi_\star$ be an optimal context-dependent policy for $\M$.
    Under the good event of \cref{corl:UCB-value-main}, we have that for any $t \geq 1$
    \begingroup \allowdisplaybreaks
    \begin{align*}
        \E_c\left[V^{\pi_\star(c;\cdot)}_{\M(c)}\right]
        \leq &
        \E_c \left[ V^{\pi_t(c;\cdot)}_{\Mhat_t(c)}(s_0)\right]
        +
        \frac{2029 H^4 \dEP}{\betaP}
        \log^2(8TH  |\Fp|/\delta)
        \\
        &
        \hspace{7em}
        +
        \frac{504 H^3 \dEP }{\betaR} 
        \log^2(64T^4 H  |\F| |\Fp|/\delta^2)
        .
    \end{align*}
    \endgroup
\end{lemma}

\begin{proof}
    Fix any round $t \geq 1$ consider the following derivation.
    \begingroup\allowdisplaybreaks
    \begin{align*}
        \E_c\left[V^{\pi_\star(c;\cdot)}_{\M(c)}(s_0)\right]
        &
        = 
        \E_c\left[V^{\pi_\star(c;\cdot)}_{\M^{(f_\star,P_\star)}(c)}(s_0)\right]
        \\
        \tag{By \cref{corl:UCB-value-main}}
        &
        \leq 
        \E_c[V^{\pi_\star(c;\cdot)}_{\M^{(\widehat{f}_t, \widehat{P}_t)}(c)}(s_0)]
        \\
        &
        \hspace{1em}
        +
        \E_c 
        \left[
        \sum_{h=0}^{H-1}
        \sum_{s_h \in S_h^c}
        \sum_{a_h \in A}
        q_h(s_h,a_h| \pi(c;\cdot), \widehat{P}^c_t )
        \brk*{
        b_{t,h}^R(c,s_h,a_h)
        +
        b_{t,h}^P(c,s_h,a_h)
        }
        \right]
        \\
        &
        \hspace{1em}
        +
        \frac{2029 H^4 \dEP}{\betaP}
        \log^2(8TH  |\Fp|/\delta)
        +
        \frac{504 H^3 \dEP }{\betaR} 
        \log^2(64T^4 H  |\F| |\Fp|/\delta^2)
        \\
        \tag{\cref{eq:value-occupancy-representation}}
        = &
        \E_c[V^{\pi_\star(c;\cdot)}_{\Mhat_t(c)}(s_0)]
        +
        \frac{2029 H^4 \dEP}{\betaP}
        \log^2(8TH  |\Fp|/\delta)
        \\
        &
        \hspace{1em}
        +
        \frac{504 H^3 \dEP }{\betaR} 
        \log^2(64T^4 H  |\F| |\Fp|/\delta^2)
        \\
        \tag{$\pi_t$ is optimal in $\Mhat_t$}
        \leq &
        \E_c \left[ V^{\pi_t(c;\cdot)}_{\Mhat_t(c)}(s_0)\right]
        +
        \frac{2029 H^4 \dEP}{\betaP}
        \log^2(8TH  |\Fp|/\delta)
        \\
        &
        \hspace{1em}
        +
        \frac{504 H^3 \dEP }{\betaR} 
        \log^2(64T^4 H  |\F| |\Fp|/\delta^2)
        ,
    \end{align*}
    \endgroup
    as the lemma states.
\end{proof}

\begin{lemma}[restatement of \cref{lemma:optimism-cost-main}]
    Under the good event of \cref{corl:UCB-value-main}, we have that for every $t \geq 1$
    \begin{align*}
        \E_c\left[ V^{\pi_t(c;\cdot)}_{\Mhat_t(c)}(s_0)\right]
        \leq
        \E_c\left[ V^{\pi_t(c;\cdot)}_{\M(c)}(s_0)\right]
        &
        +
        2 \sum_{h=0}^{H-1}\E_{c} \brk[s]*{
        \mathop{\E}_{\pi_t(c;\cdot), \widehat{P}_t^c} \brk[s]*{
        b^R_{t,h}(c,s_h,a_h) + b^P_{t,h}(c,s_h,a_h)
        }
        }
        \\
        &
        +
        \frac{2029 H^4 \dEP}{\betaP}
        \log^2(8TH  |\Fp|/\delta)
        \\
        &
        +
        \frac{504 H^3 \dEP }{\betaR} 
        \log^2(64T^4 H  |\F| |\Fp|/\delta^2)
        .
    \end{align*}
\end{lemma}
\begin{proof}
    For all $t \ge 1$ the following holds.
    \begingroup\allowdisplaybreaks
    \begin{align*}
       &\E_c\left[ V^{\pi_t(c;\cdot)}_{\Mhat_t(c)}(s_0)\right]
       \\
       \tag{\cref{eq:value-occupancy-representation}}
       &
       =
       \E_c\left[
        \sum_{h=0}^{H-1} \sum_{s \in S^c_h}\sum_{a_h \in A} q_h(s,a|\pi_t(c;\cdot),\widehat{P}^c_t)
        \cdot \left( \widehat{f}_t (c,s_h,a_h) + b^R_{t,h}(c,s_h,a_h) +  b^P_{t,h}(c,s_h,a_h) \right) \right]
        \\
        &
        =
        \E_c\left[ V^{\pi_t(c;\cdot)}_{\M^{(\widehat{f}_t, \widehat{P}_t)}(c)}(s_0)\right]
        \\
        &
        \qquad
        +
        \E_c\left[
        \sum_{h=0}^{H-1}
        \sum_{s_h \in S^c_h} \sum_{a_h \in A}
        q_h(s_h,a_h|\pi_t(c;\cdot),\widehat{P}^c_t)
        \brk*{
        b_{t,h}^R(c,s_h,a_h)
        +
        b_{t,h}^P(c,s_h,a_h)
        }
        \right]
        \\
        \tag{\cref{corl:UCB-value-main}}
        &
        \le
        \E_c\left[ V^{\pi_t(c;\cdot)}_{\M^{({f}_\star, {P}_\star)}(c)}(s_0)\right]
        \\
        &
        \qquad
        +
        2
        \E_c\left[
        \sum_{h=0}^{H-1}
        \sum_{s_h \in S^c_h} \sum_{a_h \in A}
        q_h(s_h,a_h|\pi_t(c;\cdot),\widehat{P}^c_t)
        \brk*{
        b_{t,h}^R(c,s_h,a_h)
        +
        b_{t,h}^P(c,s_h,a_h)
        }
        \right]
        \\
        &
        \qquad
        +
        \frac{2029 H^4 \dEP}{\betaP}
        \log^2(8TH  |\Fp|/\delta)
        +
        \frac{504 H^3 \dEP }{\betaR} 
        \log^2(64T^4 H  |\F| |\Fp|/\delta^2)
        ,
    \end{align*}
    \endgroup
    and the proof follows by writing the second term as an expectation over $s_h,a_h$ when playing $\pi_t(c; \cdot)$ on the dynamics $\widehat{P}_t^c$.
\end{proof}

\subsection{Regret Bound}
\label{sec:regret-proof}
We begin with a technical result that bounds the expected cumulative bonuses.
\begin{lemma}\label{lemma:cummulative-bonuses}
    Let $b^R_{t,h}(c,s_h,a_h), b^P_t(c,s_h,a_h)$ be the reward bonuses in \cref{eq:reward-bonuses}. We have that
    \begin{align*}
        \sum_{t=1}^{T}\sum_{h=0}^{H-1}\E_{c} \brk[s]*{
        \E_{\pi_t(c;\cdot), \widehat{P}_t^c} \brk[s]*{
        b^R_{t,h}(c,s_h,a_h) + b^P_t(c,s_h,a_h)
        }
        }
        \le
        H|S||A| \brk{\betaR + H \betaP} \log(T+1)
        .
    \end{align*}
\end{lemma}
\begin{proof}
    First we bound $b^R_{t,h}(c,s_h,a_h), b^P_{t,h}(c,s_h,a_h)$ by the second term in the minimum to get that
    \begingroup\allowdisplaybreaks
    \begin{align*}
        &\sum_{t=1}^{T}\sum_{h=0}^{H-1}
        \E_{c} 
        \brk[s]*{
        \E_{\pi_t(c;\cdot), \widehat{P}_t^c} \brk[s]*{
        b^R_{t,h}(c,s_h,a_h) + b^P_{t,h}(c,s_h,a_h)
        }
        }
        \\
        &
        =
        \E_c\brk[s]*{
        \sum_{t=1}^{T}
        \sum_{h=0}^{H-1}
        \sum_{s_h \in S^c_h}
        \sum_{a_h \in A} q_h(s_h,a_h|\pi_t(c;\cdot),\widehat{P}^c_t)
        \brk[s]*{
        b^R_{t,h}(c,s_h,a_h) + b^P_{t,h}(c,s_h,a_h)
        }
        }
        \\
        &\le
        \frac12 \brk*{\betaR + \betaP H}
        \E_c\brk[s]*{
        \sum_{t=1}^{T}
        \sum_{h=0}^{H-1}
        \sum_{s_h \in S^c_h}
        \sum_{a_h \in A} \frac{q_h(s_h,a_h|\pi_t(c;\cdot),\widehat{P}^c_t)}{1 + \sum_{i=1}^{t-1}q_h(s_h,a_h| \pi_i(c;\cdot), \widehat{P}^c_i)}
        }
        \\
        &
        =
        \frac12 \brk*{\betaR + \betaP H}
        \E_c\brk[s]*{
        \sum_{h=0}^{H-1}
        \sum_{s_h \in S^c_h}
        \sum_{a_h \in A} 
        \sum_{t=1}^{T}
        \frac{q_h(s_h,a_h|\pi_t(c;\cdot),\widehat{P}^c_t)}{1 + \sum_{i=1}^{t-1}q_h(s_h,a_h| \pi_i(c;\cdot), \widehat{P}^c_i)}
        }
        \\
        &
        \le
        \frac12 \brk*{\betaR + \betaP H}
        \E_c\brk[s]*{
        \sum_{h=0}^{H-1}
        \sum_{s_h \in S^c_h}
        \sum_{a_h \in A} 
        2 \log (T+1)
        }
        \\
        &
        =
        H \abs{S} \abs{A} \brk*{\betaR + \betaP H} \log (T+1)
        ,
    \end{align*}
    \endgroup
    where the last inequality used \cref{lemma:general-potential-bound} with $x_t = q_h(s_h,a_h|\pi_t(c;\cdot),\widehat{P}^c_t)$.
\end{proof}

Using all the above, we derive our main result stated in the following theorem.
\begin{theorem}[E-UC$^3$RL regret bound, restatement of~\cref{thm:UC3RL-regret-bound}]
    For any $T > 1$ and $\delta \in (0,1)$, suppose we run \cref{alg:UCCRL-main} with parameters 
    \begin{align*}
        \betaR
        =
        \sqrt{\frac{504 T H^2 \dEP \log^2(64 T^4 H \abs{\F} |\Fp|/\delta^2)}{|S||A| \log(T+1)}}
        ,
        \quad
        \betaP
        =
        \sqrt{\frac{2029 T H^2 \dEP \log^2(8 T H |\Fp| /\delta)}{|S||A| \log(T+1)}}
        ,
    \end{align*}
    and $\dEP\geq \dHE(\Fp, D_H, T^{-1/2})$.
    Then, with probability at least $1-\delta$
    \begin{align*}
        \Regrv_T(\text{E-UC$^3$RL}) 
        \leq
        \widetilde{O}
        \brk*{
        H^3
        \sqrt{
        T \abs{S} \abs{A}
        \dEP
        \left(\log \brk{\abs{\F}/\delta}  + \log \brk{\abs{\Fp}/ \delta} \right)
        }
        }
        .
    \end{align*}
\end{theorem}

\begin{proof}[Proof of~\cref{thm:UC3RL-regret-bound}]
    We prove a regret bound under the following good events.
    The first event is that of \cref{corl:UCB-value-main}, which occurs with probability at least $1-3\delta/4$. 
    The second event is that
    \begin{equation}\label{eq:regret-to-expected-regret}
        \begin{split}
            \sum_{t=1}^T V^{\pi^\star(c_t; \cdot)}_{\M(c_t)}(s_0) -  V^{\pi_t(c_t; \cdot)}_{\M(c_t)}(s_0)
            \leq 
            \sum_{t=1}^T \E_c \left[V^{\pi^\star(c; \cdot)}_{\M(c)}(s_0)\right] - \E_c\left[V^{\pi_t(c; \cdot)}_{\M(c)}(s_0)\right] 
            +
            H\sqrt{T \log(8/\delta)}.   
        \end{split}
    \end{equation}
    By Azuma's inequality (where the filtration is the histories $\{\Hist_t\}_{t=1}^T$), the above holds with probability at least $1-\delta/4$. Taking a union bound, the good event holds with probability at least $1-\delta$.
    Hence, assume the good events hold, and consider the following derivation.
     \begingroup\allowdisplaybreaks
    \begin{align*}
        &
        \sum_{t=1}^T \E_c \left[V^{\pi^\star(c; \cdot)}_{\M(c)}(s_0)\right] - \E_c\left[V^{\pi_t(c; \cdot)}_{\M(c)}(s_0)\right]
        \\
        = &
        \sum_{t=1}^T 
        \E_c \left[V^{\pi^\star(c; \cdot)}_{\M(c)}(s_0)\right]
        -
        \E_c\left[V^{\pi_t(c; \cdot)}_{\Mhat_t(c)}(s_0)\right]
        \\
        & +
        \sum_{t=1}^T 
        \E_c \left[V^{\pi_t(c; \cdot)}_{\Mhat_t(c)}(s_0)\right] - \E_c\left[V^{\pi_t(c; \cdot)}_{\M(c)}(s_0)\right]
        \\
        \tag{\cref{lemma:opt-in-expectation-main,lemma:optimism-cost-main}}
        \le &
        2\sum_{t=1}^{T} \sum_{h=0}^{H-1}\E_{c} \brk[s]*{
        \mathop{\E}_{\pi_t(c;\cdot), \widehat{P}_t^c} \brk[s]*{
        b^R_{t,h}(c,s_h,a_h) + b^P_{t,h}(c,s_h,a_h)
        }
        }
        \\
        &
        +
        2T
        \frac{2029 H^4 \dEP}{\betaP}
        \log^2(8TH  |\Fp|/\delta)
        \\
        &
        +
        2T
        \frac{504 H^3 \dEP }{\betaR} 
        \log^2(64T^4 H  |\F| |\Fp|/\delta^2)
        \\
        \tag{\cref{lemma:cummulative-bonuses}}
        \le &
        2 H \abs{S} \abs{A} \brk{\betaR + H \betaP}
        \log (T+1)
        \\
        &
        +
        2T
        \frac{2029 H^4 \dEP}{\betaP}
        \log^2(8TH  |\Fp|/\delta)
        \\
        &
        +
        2T
        \frac{504 H^3 \dEP }{\betaR} 
        \log^2(64T^4 H  |\F| |\Fp|/\delta^2)
        \\
        \tag{Plugging in $\betaP,\betaR$}
        = &
        4 \sqrt{504 T \abs{S}\abs{A} H^4 \dEP \log(T+1) \log^2(64 T^4 H \abs{\F} \abs{\Fp} / \delta^2}
        \\
        &
        +
        4 \sqrt{2029 T \abs{S}\abs{A} H^6 \dEP \log(T+1) \log^2(8 T H \abs{\Fp} / \delta}
        \\
        \le &
        4 \sqrt{2029 T \abs{S}\abs{A} H^6 \dEP \log(T+1) \log^2(8 T H \abs{\Fp} / \delta}
        \\
        &
        270 H^3
        \sqrt{ T \abs{S} \abs{A}
        \dEP
        \log (T+1)
        \log^2 \brk{18 T^4 H \abs{\F}\abs{\Fp}/ \delta^2}
        }
        .
    \end{align*}
     \endgroup
    Finally, we get that
    \begingroup
    \allowdisplaybreaks
    \begin{align*}
        \Regrv_T(\text{E-UC$^3$RL})
        = &
        \sum_{t=1}^T V^{\pi^\star(c_t; \cdot)}_{\M(c_t)}(s_0) -  V^{\pi_t(c_t; \cdot)}_{\M(c_t)}(s_0)
        \\
        \tag{\cref{eq:regret-to-expected-regret}}
        \leq &
        \sum_{t=1}^T \brk*{\E_c \left[V^{\pi^\star(c; \cdot)}_{\M(c)}(s_0)\right] - \E_c\left[V^{\pi_t(c; \cdot)}_{\M(c)}(s_0)\right]} 
        \\
        & +
        H\sqrt{T \log(8/\delta)}
        \\
        \leq &
        271 H^3
        \sqrt{ T \abs{S} \abs{A}
        \dEP
        \log (T+1)
        \log^2 \brk{18 T^4 H \abs{\F}\abs{\Fp}/ \delta^2}
        }
        \\
        = &
        \widetilde{O}
        \brk*{
        H^3
        \sqrt{
        T \abs{S} \abs{A}
        \dEP
        \left(\log \brk{\abs{\F}/\delta}  + \log \brk{\abs{\Fp}/ \delta} \right)
        }
        }
        .
        \qedhere
    \end{align*}
    \endgroup
\end{proof}

\section{Auxiliary lemmas}\label{Appendix:Aux-lemmas}

\begin{lemma}\label{lemma:general-potential-bound}
    Let $S_t = \lambda + \sum_{k=1}^{t-1} x_t$. Suppose $x_t \in [0,\lambda]$ and , then
    \begin{align*}
        \sum_{t=1}^{T} \frac{x_t}{S_t}
        \le
        2 \log (T+1)
        .
    \end{align*}
\end{lemma}
\begin{proof}
    The following holds.
    \begingroup\allowdisplaybreaks
    \begin{align*}
        \sum_{t=1}^{T} \frac{x_t}{S_t}
        &
        =
        \sum_{t=1}^{T} \frac{S_{t+1} - S_t}{S_t}
        \\
        &
        =
        \sum_{t=1}^{T} \frac{S_{t+1}}{S_t} - 1
        \\
        &
        \tag{ $1 \leq \frac{S_{t+1}}{S_t} \leq 2$ since $x_t \leq \lambda$}
        \le
        2 \sum_{t=1}^{T} \log\frac{S_{t+1}}{S_t}
        \\
        &
        =
        2\sum_{t=1}^{T}\log S_{t+1} - \log S_t
        \\
        \tag{telescopic sum}
        &
        =
        2 \log \frac{S_{T+1}}{S_1}
        \\
        &
        \le
        2 \log (T+1)
        .
        \qedhere
    \end{align*}
    \endgroup
\end{proof}

\begin{lemma}[value-difference, Corollary $1$ in~\citealp{efroni2020optimistic}]\label{lemma:val-diff}
    Let $M$, $M'$ be any $H$-finite horizon MDPs. Then, for any two policies $\pi$, $\pi'$ the following holds
    \begingroup\allowdisplaybreaks
    \begin{align*}
        &
        V^{\pi,M}_0(s)
        -
        V^{\pi',M'}_0(s) 
        =
        \sum_{h=0}^{H-1} \E \left[ \langle  Q^{\pi,M}_h(s_h, \cdot) , \pi_h(\cdot|s_h) - \pi'_h(\cdot|s_h) \rangle |s_0 = s, \pi',M' \right]
        \\
        &\hspace{2em} +
        \sum_{h=0}^{H-1} \E \left[ 
        r_h(s_h,a_h) -  r'_h(s_h,a_h)
        + (p_h(\cdot|s_h,a_h)-p'_h(\cdot|s_h,a_h)) V^{\pi,M}_{h+1}
        |s_h = s, \pi',M' \right] .       
    \end{align*}
    \endgroup
\end{lemma}







\end{document}